\documentclass[preprint,10pt,3p]{elsarticle}
\usepackage{geometry}
\usepackage{booktabs}
\usepackage[table]{xcolor}
\usepackage{tabularx}
\usepackage{graphicx}
\usepackage{multirow}
\usepackage{caption}
\usepackage{amsmath}
\usepackage{amssymb}
\usepackage{algorithm}
\usepackage{algpseudocode}

\usepackage{amsfonts}
\usepackage{subfigure}
\usepackage{times}
\usepackage{enumitem}
\usepackage{amsthm}
\usepackage{todonotes}
\usepackage[colorlinks=false]{hyperref}

\newtheorem{definition}{Definition}

\newtheorem{theorem}{Theorem}

\begin{document}

\begin{frontmatter}

\title{CurricuVLM: Towards Safe Autonomous Driving via Personalized Safety-Critical Curriculum Learning with Vision-Language Models}

\author[a]{Zihao Sheng\textsuperscript{†}}
\ead{zihao.sheng@wisc.edu}
\author[a]{Zilin Huang\textsuperscript{†}}
\ead{zilin.huang@wisc.edu}
\author[b]{Yansong Qu}
\ead{qu120@purdue.edu}
\author[c]{Yue Leng}
\ead{lengc@google.com}
\author[c]{Sruthi Bhavanam}
\ead{sruthibhavanam@google.com}
\author[a]{Sikai Chen\corref{cor}}
\ead{sikai.chen@wisc.edu}

\cortext[cor]{Corresponding author: Sikai Chen. {†} Equal contribution.}

\address[a]{Department of Civil and Environmental Engineering, University of Wisconsin-Madison, Madison, WI 53706, USA}
\address[b]{Lyles School of Civil and Construction Engineering, Purdue University, West Lafayette, IN 47907, USA} 
\address[c]{Google, Sunnyvale, CA 94089, USA} 

\begin{abstract}
Ensuring safety in autonomous driving systems remains a critical challenge, particularly in handling rare but potentially catastrophic safety-critical scenarios. While existing research has explored generating safety-critical scenarios for autonomous vehicle (AV) testing, there is limited work on effectively incorporating these scenarios into policy learning to enhance safety. Furthermore, developing training curricula that adapt to an AV's evolving behavioral patterns and performance bottlenecks remains largely unexplored. To address these challenges, we propose CurricuVLM, a novel framework that leverages Vision-Language Models (VLMs) to enable personalized curriculum learning for autonomous driving agents. Our approach uniquely exploits VLMs' multimodal understanding capabilities to analyze agent behavior, identify performance weaknesses, and dynamically generate tailored training scenarios for curriculum adaptation. Through comprehensive analysis of unsafe driving situations with narrative descriptions, CurricuVLM performs in-depth reasoning to evaluate the AV's capabilities and identify critical behavioral patterns. The framework then synthesizes customized training scenarios targeting these identified limitations, enabling effective and personalized curriculum learning. Extensive experiments on the Waymo Open Motion Dataset show that CurricuVLM outperforms state-of-the-art baselines across both regular and safety-critical scenarios, achieving superior performance in terms of navigation success, driving efficiency, and safety metrics. Further analysis reveals that CurricuVLM serves as a general approach that can be integrated with various RL algorithms to enhance autonomous driving systems. 
The code and demo video are available at: \href{https://zihaosheng.github.io/CurricuVLM/}{\textcolor{magenta}{https://zihaosheng.github.io/CurricuVLM/}}.
\end{abstract}

\begin{keyword}
Autonomous Driving, Vision-Language Models, Safety-Critical Systems, Curriculum Learning, Reinforcement Learning
\end{keyword}

\end{frontmatter}

\section{Introduction}

The advancement of artificial intelligence and sensor technologies has enabled autonomous vehicles (AVs) to achieve remarkable progress over the past decades, particularly in perception, decision-making, and control capabilities \citep{di2021survey,wu2024recent,sheng2024traffic,huang2024toward}. These advancements have facilitated the transition of autonomous driving from theoretical research to preliminary commercial deployment, with leading technology companies like Waymo and Cruise initiating pilot autonomous taxi services in selected cities across the United States \citep{kusano2024comparison,california_dmv_disengagement_2021}. While AVs promise to revolutionize transportation with improved mobility and reduced emissions, their widespread deployment remains constrained by persistent safety-related challenges \citep{ding2023survey,cao2022trustworthy,huang2024human,he2024trustworthy}. 
Real-world traffic environments are inherently complex, presenting a “long tail” of rare but critical scenarios, such as sudden pedestrian crossings or abrupt loss of control by nearby vehicles. These scenarios, while infrequent, often carry significant consequences and present substantial challenges. Due to their urgency and safety-critical nature, instantaneous decision-making is demanded, where subtle misjudgments could lead to fatal outcomes. Even human drivers require extensive training to handle them reliably. In light of these challenges, enhancing AV safety and robustness in these safety-critical scenarios has emerged as a key research focus.

The current literature on autonomous driving safety enhancement broadly divides into two categories: policy-level approaches that focus on designing safer driving policies and scene-level approaches that generate diverse safety-critical scenarios to assess an AV agent's robustness \citep{zhang2023cat}. Regarding policy-level approaches, a common paradigm in this category involves combining traditional reinforcement learning (RL) with specific safety constraints to ensure safer policy outcomes \citep{zhu2020safe,li2022decision}. For instance, constrained Markov Decision Processes (CMDP) and hierarchical RL frameworks are often employed to impose safety guarantees, dynamically adjusting exploration to avoid unsafe states \citep{gu2023safe,isele2018safe}. Recent works introduce trustworthy physics knowledge and interpretable rule-based policies as controllable safety performance lower bounds, ensuring that learned RL policies can not only adapt to complex scenarios but also achieve statistically guaranteed improvements over existing baseline policies \citep{cao2022trustworthy,huang2024trustworthy}. Orthogonal to policy-level approaches, scene-level methods focus on generating diverse scenarios to evaluate AV agents' performance, particularly in rare, accident-prone situations \citep{ma2024evolving}. In recent years, the emergence of advanced generative models, particularly diffusion models, has revolutionized the generation of realistic and complex testing scenarios \citep{zhong2023guided,lu2024scenecontrol}. These models excel at synthesizing diverse accident-prone events with high fidelity, enabling comprehensive stress testing of AV agents under controlled conditions. Such systematic evaluation is crucial for identifying and addressing potential policy weaknesses before real-world deployment \citep{feng2023dense}.

\begin{figure}
  \centerline{\includegraphics[width=0.99\textwidth]{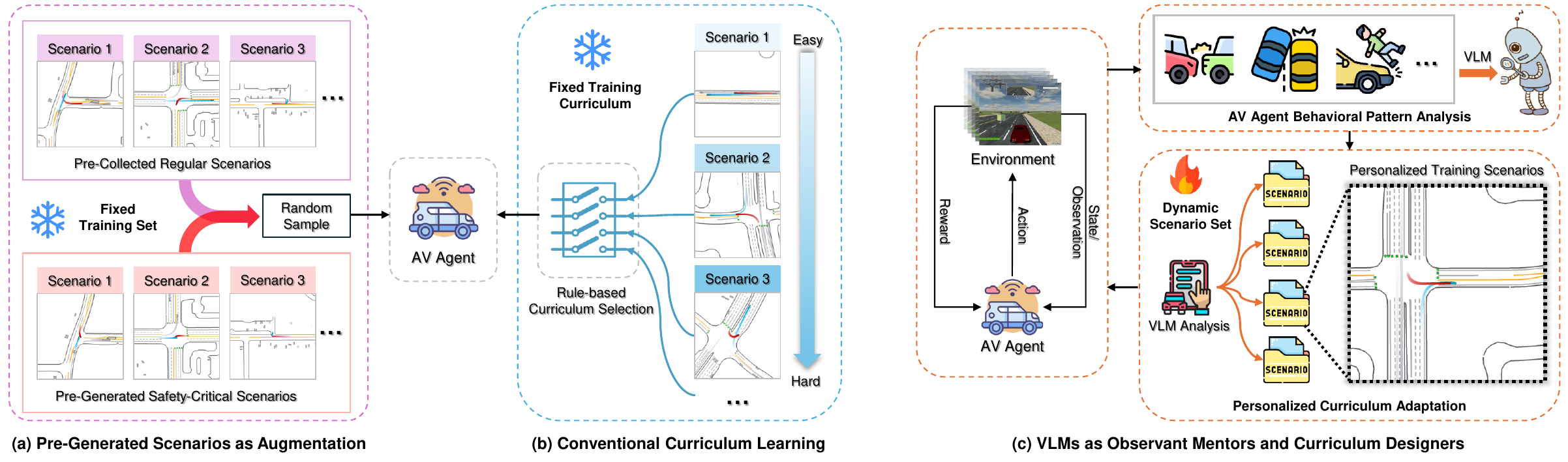}}
  \caption{Comparison of different approaches for safety-critical scenario integration in autonomous driving. (a) Pre-generated safety-critical scenarios as static training augmentation, which fails to adapt to the agent's evolving capabilities; (b) Conventional curriculum learning with rule-based scenario selection, which lacks personalization to individual learning bottlenecks; (c) Our proposed approach using VLMs as observant mentors and curriculum designers for dynamic and personalized training.}
  \label{motivation}
\end{figure}

Despite significant advances in scenario generation and testing, two main challenges remain in effectively integrating these methods into the closed-loop training for AV agents. Firstly, a dynamic mismatch exists between the static training scenario distribution and the evolving capabilities of AV agents. As shown in Fig.~\ref{motivation}(a), a common approach in recent studies is to pre-generate a fixed set of safety-critical scenarios as augmented training samples \citep{zhang2024chatscene}. However, as the AV agent's proficiency grows through learning, these initially challenging scenarios may become insufficient for effectively examining performance boundaries, thereby limiting their contribution to the agent's continual improvement. Secondly, designing an appropriate training curriculum remains challenging. Research in curriculum learning has demonstrated that intelligent agents, including AV agents, benefit from progressive learning that transitions from simple to complex scenarios \citep{wang2021survey}. A critical requirement in this process is that the difficulty level of selected scenarios must align with the agent's current capacity. If scenarios are too challenging or too simple, training efficiency can decline due to optimization failures in intractable situations or learning saturation in trivial cases. As illustrated in Fig.~\ref{motivation}(b), recent efforts have explored solving this issue by ordering scenarios based on predefined rules or heuristics, such as the number of interactive agents or complexity of road conditions \citep{anzalone2022end,wu2024recent}. However, the static manually-crafted metrics often struggle to adapt to the agent's evolving proficiency levels and individual learning bottlenecks across different driving skills.

Recently, the emergence of Large Language Models (LLMs) and Vision-Language Models (VLMs) has opened new avenues for enhancing autonomous driving through their powerful capabilities in understanding complex visual scenes and reasoning within open-world environments \citep{zhou2024vision,cui2024survey,huang2024vlm}. Unlike traditional computer vision approaches that often focus on individual tasks such as detection or captioning, VLMs leverage Internet-scale multimodal pretraining to develop comprehensive visual understanding and reasoning capabilities. Several recent studies have shown VLMs' potential in driving scene understanding, especially in identifying and explaining safety-critical situations \citep{xiao2024hazardvlm,shi2024scvlm}. Their capacity to generate detailed natural language descriptions of driving scenarios makes them ideal tools for analyzing and explaining agent behavior patterns. Despite these advantages, the integration of VLMs into closed-loop autonomous driving training pipelines, particularly in the context of safety-critical scenarios, remains largely unexplored.

Building upon these insights, we propose CurricuVLM, a novel framework that leverages the potential of VLMs for personalized safety-critical curriculum learning in autonomous driving. Inspired by how human instructors observe, analyze, and design targeted exercises for student drivers, our framework addresses the aforementioned challenges by establishing a closed-loop training pipeline where VLMs serve as both observant mentors and curriculum designers, as illustrated in Fig.~\ref{motivation}(c). Specifically, CurricuVLM continuously monitors the agent's behavior in various driving scenarios, using VLMs to understand and analyze driving situations. These observations are then processed through a specialized GPT-4o-based analyzer that assesses the agent's current capabilities and limitations. Based on this comprehensive analysis, our framework dynamically generates personalized training scenarios that are both challenging and instructive, precisely tailored to the agent's evolving capabilities.

The main contributions of our work can be summarized as follows:
\begin{enumerate}[label=\arabic*)]
    \item We propose a novel framework named CurricuVLM that bridges the gap between scenario generation and driving policy learning by leveraging VLMs to enable personalized safety-critical curriculum learning. \textbf{To the best of our knowledge, this is the first work to utilize VLMs for dynamic curriculum generation in closed-loop autonomous driving training.}
    \item Building upon the CurricuVLM framework, we design a novel two-stage behavior analysis pipeline that combines VLMs' visual understanding capabilities with GPT-4o's reasoning abilities. This approach identifies the agent's performance bottlenecks and generates informed recommendations for curriculum adaptation to align with the agent's evolving capabilities.
    \item To facilitate effective training, we introduce a novel scenario generation approach guided by these personalized recommendations. This approach allows CurricuVLM to create diverse and realistic safety-critical scenarios that specifically target the agent's identified weaknesses, enabling more effective closed-loop training to improve AV safety performance progressively.
    \item Through extensive experiments on the Waymo Open Motion Dataset, we demonstrate CurricuVLM's superior performance over state-of-the-art methods in terms of navigation success, driving efficiency, and safety metrics across both regular and safety-critical scenarios. Our approach shows strong compatibility with various RL algorithms including TD3, PPO, and SAC, demonstrating its potential as a general framework for enhancing autonomous driving systems.
\end{enumerate}

The rest of this paper is organized as follows: Section \ref{sec2} reviews related work in autonomous driving safety, curriculum learning, and language models. Section \ref{sec3} presents the problem formulation and formally defines the key concepts in safety-critical curriculum learning for autonomous driving.  Section \ref{sec4} presents the detailed methodology of our CurricuVLM framework. Section \ref{sec5} presents comprehensive experimental results on the Waymo Open Motion Dataset. Finally, Section \ref{sec6} summarizes the paper and concludes with a discussion of future research directions.

\section{Related works}\label{sec2}
\subsection{Safety-Critical Driving Scenario Generation}

In recent years, research on generating safety-critical driving scenarios has gained attention as a means to evaluate and validate autonomous driving models. This field has attracted significant attention from various stakeholders, with regulatory bodies like NHTSA establishing frameworks for testable cases and pre-crash scenarios \citep{najm2007pre,thorn2018framework}, while industry leaders such as Waymo focusing on reconstructing real-world fatal crashes in simulations \citep{scanlon2021waymo}. Within the academic community, researchers have explored various approaches to advance this field \citep{ding2023survey}. For instance, \cite{ding2020learning} reformulated scenario generation as a reinforcement learning problem, where the generation model acts as an agent and the driving algorithm to be evaluated serves as the environment. SceneGen \citep{tan2021scenegen} proposed a neural autoregressive model to sequentially insert actors into the scene while synthesizing their states. However, SceneGen only focuses on generating static traffic scene snapshots without considering subsequent vehicle trajectories. STRIVE \citep{rempe2022generating} addressed this limitation by introducing a learned traffic motion model to generate safety-critical vehicle trajectories. More recently, TrafficGen \citep{feng2023trafficgen} developed an autoregressive neural architecture to generate both realistic initial states and long-term trajectories. To improve the controllability of scenario generation, diffusion models with guided sampling \citep{zhong2023guided,lu2024scenecontrol} have been leveraged to enable flexible control over generated scenes. Despite these advances, most existing approaches focus primarily on scenario generation and testing without effectively utilizing generated scenarios for closed-loop training of AV agents.

LLMs have recently emerged as powerful tools for enhancing traffic scenario generation through their superior natural language understanding and reasoning capabilities. CTG++ \citep{zhong2023language} pioneered this direction by employing LLMs to translate user queries into differentiable loss functions, which then guide a scene-level diffusion model to generate query-compliant traffic scenarios. Along similar lines, ChatScene \citep{zhang2024chatscene} utilized LLMs to generate textual descriptions for safety-critical scenarios, which are then decomposed into sub-descriptions for behavior, geometry, and spawn positions. These descriptions are matched with a pre-built code snippet database to generate executable Scenic code for CARLA simulation. TTSG \citep{ruan2024traffic} leveraged LLMs for road selection and agent planning, supporting diverse traffic scene generation without requiring predefined spawning points. These works demonstrate the potential of LLMs in autonomous driving, especially their capabilities to understand complex traffic contexts and generate corresponding descriptions. Inspired by these advances, our work aims to further leverage VLMs and LLMs to analyze agent behavior patterns, enabling personalized safety-critical training curricula to enhance autonomous driving safety.

\subsection{Curriculum Learning in Autonomous Driving}

Inspired by the way humans learn progressively from simpler to more complex tasks, curriculum learning \citep{bengio2009curriculum} involves structuring the learning process in stages based on task difficulty or the learner’s capabilities. This methodology has shown great potential in accelerating training and improving performance across various machine learning tasks \citep{wang2021survey}. In autonomous driving, curriculum learning has emerged as a promising approach to tackle the challenge of learning complex driving behaviors \citep{qiao2018automatically,ozturk2021investigating,khaitan2022state,peng2024reward}. For instance, \cite{anzalone2022end} proposed an end-to-end curriculum learning framework that divides reinforcement learning into multiple stages of increasing difficulty to guide the agent towards better driving policies. \cite{peiss2023graph} further demonstrated how to design effective curricula by gradually increasing task complexity in terms of traffic density, driving routes, and spatial constraints. Their curriculum starts with basic skills like lane keeping in sparse traffic, then progressively introduces more complex tasks such as collision avoidance and lane change in denser scenarios. 
However, most works rely on manual curriculum design, which may require significant expert knowledge and make it challenging to generalize across different driving scenarios. To address these limitations, recent research has begun exploring automated curriculum generation. Most notably, CLIC proposed by \cite{niu2024continual} developed a continual driving policy optimization framework with closed-loop individualized curricula. CLIC trains a discriminator network to predict collision probabilities in different scenarios and leverages these predictions to customize individualized curricula for the current AV's capability level.
Different from their work which relies solely on a black-box collision prediction model for curriculum design, our approach leverages VLM's visual understanding and reasoning abilities to assess agent behavior and performance, providing interpretable textual analysis and enabling more comprehensive and targeted curriculum generation.

\subsection{Autonomous Driving with LLMs/VLMs}

LLMs and VLMs have demonstrated remarkable capabilities across various domains, showcasing their potential in natural language understanding, visual reasoning, and complex decision-making tasks \citep{achiam2023gpt,dubey2024llama}. 
Recently, there has been a growing interest in integrating these powerful models into autonomous driving systems to enhance their perception, reasoning, and decision-making capabilities \citep{zhou2024vision}. 
GPT-Driver \citep{mao2023gpt} pioneered this integration by incorporating LLMs into autonomous driving planning, utilizing GPT-based models to generate high-level driving decisions from natural language descriptions of traffic scenarios. 
Building upon this foundation, DiLu \citep{wen2023dilu} introduced a comprehensive framework that combines LLM-based reasoning, reflection, and memory modules, enabling decision-making based on common sense knowledge while continuously accumulating driving experience for self-reflection. 
DriveLM \citep{sima2023drivelm} proposed a novel graph visual question answering task, reformulating autonomous driving's perception, prediction, and planning processes as a sequence of graph-structured question-answering interactions. 
Taking inspiration from human cognition, LeapAD \citep{mei2024continuously} developed a dual-decision architecture that leverages LLMs for in-depth analysis and reasoning while employing lightweight models for rapid experience-based decision-making. 
Furthermore, ELM \citep{zhou2024embodied} enhanced VLMs' spatial perception and long-horizon extrapolation capabilities in driving scenarios through the utilization of large-scale open-world data.
These pioneering works inspire us to leverage the powerful reasoning capabilities of VLMs and LLMs for analyzing agent behaviors and generating personalized safety-critical scenarios.

\section{Problem Formulation}\label{sec3}

We formulate the autonomous driving task as a Markov Decision Process (MDP) \citep{wu2024recent}, defined by the tuple $(\mathcal{S}, \mathcal{A}, \mathcal{P}, R, \gamma)$. In this formulation, $\mathcal{S}$ represents the state space, encompassing all possible states of the driving environment, including the AV's own state and the states of other traffic agents. The action space $\mathcal{A}$ comprises all possible actions the AV can take, such as steering angles, throttle, and braking commands. The state transition probability function $\mathcal{P}: \mathcal{S} \times \mathcal{A} \times \mathcal{S} \rightarrow [0,1]$ defines the probability of transitioning from one state to another given an action. The reward function $R: \mathcal{S} \times \mathcal{A} \rightarrow \mathbb{R}$ assigns a scalar reward to each state-action pair, reflecting the desirability of particular actions in specific states. The discount factor $\gamma \in [0,1)$ determines the importance of future rewards relative to immediate ones. The objective of the AV agent is to learn an optimal policy $\pi^*: \mathcal{S} \rightarrow \mathcal{A}$ that maximizes the expected cumulative discounted reward:
\begin{equation}\label{eq1}
    \pi^* = \arg\max_{\pi} \mathbb{E}_{\tau \sim \mathcal{P}(\cdot|\pi)}\left[ \sum_{t=0}^\infty \gamma^t R(s_t, \pi(s_t)) \right],
\end{equation}
where $\tau = (s_0, a_0, s_1, ...)$ denotes a trajectory unrolled by policy $\pi$ under the transition dynamics $\mathcal{P}$.

To better formulate our problem of learning safe driving policies through personalized curriculum learning, we first introduce several key concepts that will be used throughout this paper:
\begin{definition}[Driving Scenario]
A driving scenario is defined as a tuple $\xi = (M, V)$, where $M$ represents the HD map containing static environment elements (e.g., road geometry, traffic signs, traffic lights), and $V = \{v_1, ..., v_N\}$ denotes the set of traffic participants and their states over temporal duration $T$. Each participant $v_i$ is characterized by its state trajectory $\{s_t^i\}_{t=1}^T$, including position, velocity, and orientation.
\end{definition}

\begin{definition}[Curriculum]\label{def2}
A curriculum $\mathcal{C} = \{\xi_1, \xi_2, ..., \xi_N\}$ is a sequence of scenarios, where each scenario $\xi_i$ can be either a regular driving scenario or a personalized safety-critical scenario based on specific learning objectives. The curriculum dynamically adapts to the agent's evolving capabilities, progressively introducing more challenging safety-critical scenarios to optimize the agent's learning process.
\end{definition}

Building upon these definitions and the base MDP formulation, we aim to enhance autonomous driving safety through effective training in safety-critical scenarios. However, a fundamental challenge in achieving this goal is the inherent rarity of safety-critical scenarios in real-world driving environments, which leads to their under-representation in training data distributions.
To address this challenge, we need to create more safety-critical scenarios and strategically inject them into a curriculum. Therefore, we propose to leverage VLMs to analyze the agent's behavior and adaptively construct personalized curricula. Let $f_{\text{VLM}}: \mathcal{O} \rightarrow \mathcal{D}$ be a VLM that maps visual observations $\mathcal{O}$ of the agent's behavior in scenarios to textual descriptions $\mathcal{D}$. These descriptions are then processed by GPT-4o through function $g_{\text{GPT-4o}}: \mathcal{D} \rightarrow \mathcal{I}$ to generate structured insights $\mathcal{I}$ about the agent's behavioral patterns and performance weaknesses. Based on these insights, we aim to construct a curriculum $\mathcal{C}$ that targets the identified weaknesses of the agent.

By learning through a carefully designed curriculum, we can extend the standard MDP objective in Eq.~\eqref{eq1} as follows:
\begin{equation}
    \pi^* = \arg\max_{\pi} \mathbb{E}_{\xi \sim \mathcal{C}(\pi, f_{\text{VLM}}, g_{\text{GPT-4o}})} \left[ \sum_{t=0}^T \gamma^t R(s_t, \pi(s_t)) \right],
\end{equation}
where $\mathcal{C}(\pi, f_{\text{VLM}}, g_{\text{GPT-4o}})$ represents our dynamic curriculum that evolves based on the agent's current policy and the VLM-based analysis. This formulation establishes a closed-loop learning system that continuously adapts to the agent's capabilities while maintaining interpretability through the VLM-based analysis framework.

\begin{figure}
  \centerline{\includegraphics[width=0.99\textwidth]{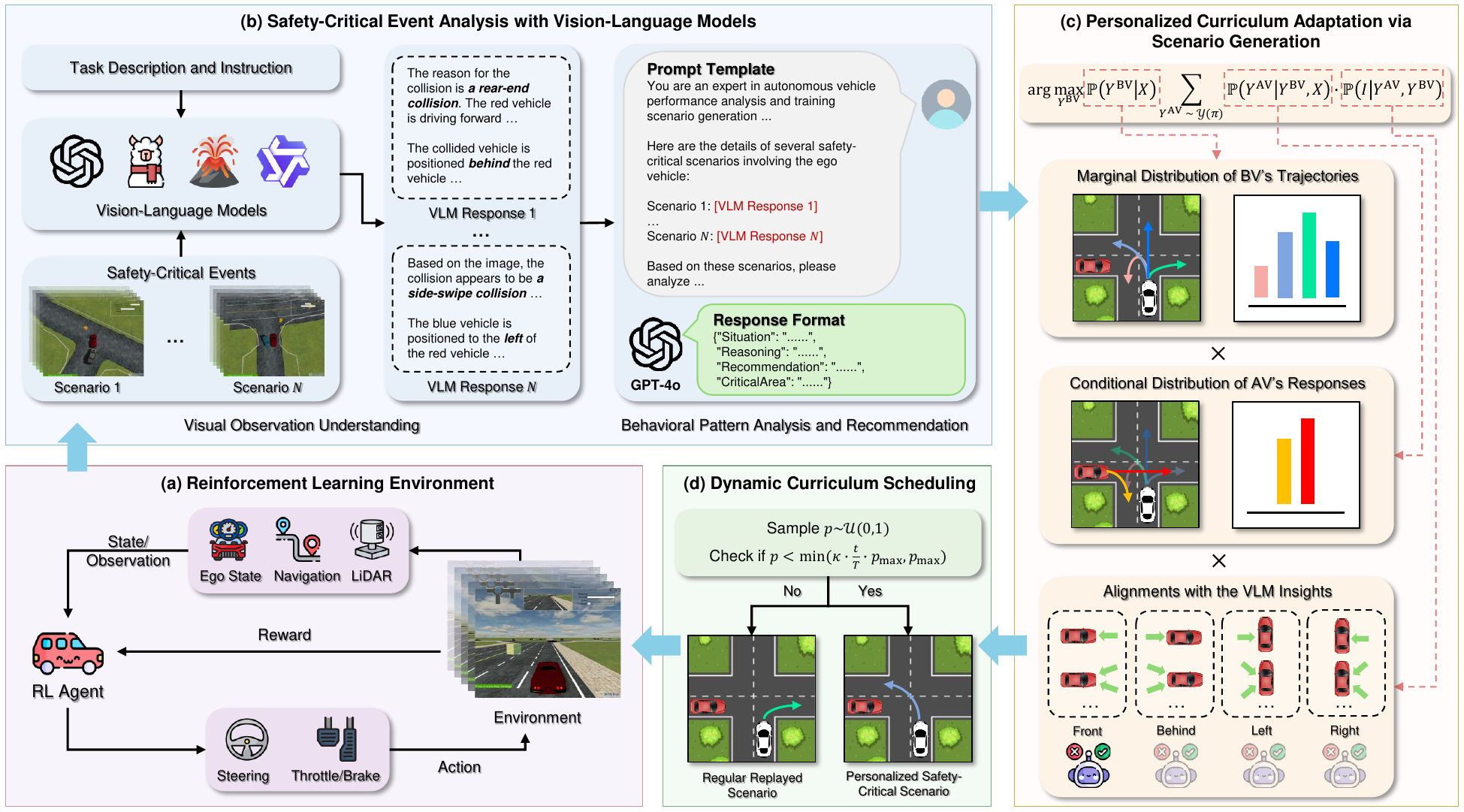}}
  \caption{Overview of CurricuVLM framework. (a) RL environment provides state observations and receives control actions from the agent; (b) Safety-critical event analysis module employs VLM for visual understanding and GPT-4o for behavioral pattern analysis; (c) Personalized curriculum adaptation generates safety-critical scenarios by optimizing background vehicle trajectories; (d) Dynamic curriculum scheduling mechanism adaptively integrates generated scenarios into the training process.}
  \label{framework}
\end{figure}

\section{Methodology}\label{sec4}

\subsection{Framework Overview}

Inspired by how human drivers improve their skills through experience and targeted practice, we propose CurricuVLM, a framework that enables autonomous driving agents to learn from their mistakes and systematically enhance their safety capabilities. Just as a driving instructor would observe a student's performance, identify recurring issues, and design specific exercises to address these weaknesses, our framework operates through three key mechanisms: (1) analyzing safety-critical events using Vision-Language Models to understand the agent's behavioral patterns, (2) generating personalized training scenarios that target identified weaknesses, and (3) dynamically integrating these scenarios into the training curriculum.

As illustrated in Fig.~\ref{framework}, the framework forms a closed-loop learning system. When the agent encounters safety-critical situations during training in Fig.~\ref{framework}(a), our VLM-based analysis module in Fig.~\ref{framework}(b) generates detailed descriptions of these events and employs GPT-4o to identify behavioral patterns and suggest improvements. These insights guide our scenario generation module in Fig.~\ref{framework}(c) in creating personalized training scenarios, which are then adaptively incorporated into the training curriculum as shown in Fig.~\ref{framework}(d). Through this iterative process, the agent progressively develops more robust driving capabilities, particularly in handling safety-critical situations.

\subsection{Safety-Critical Event Analysis with Vision-Language Models}\label{sec4.2}

\subsubsection{Visual Observation Understanding}

During each training iteration, the AV agent interacts with diverse driving scenarios from the curriculum $\mathcal{C}$, generating a series of episodes $\mathcal{E} = \{e_i\}_{i=1}^N$. Each episode $e_i$ consists of a sequence of visual observations $\mathcal{O}_i = \{o_t\}_{t=1}^T$, where $o_t$ represents the visual observation of the environment at time step $t$. These episodes become safety-critical when safety violations occur, either arising from the agent's suboptimal behavior leading to traffic violations or collisions, or emerging from challenging interactions with surrounding vehicles in our designed curriculum scenarios.

To systematically analyze the agent's behavior and identify performance weaknesses, we employ a Vision-Language Model $f_{\text{VLM}}$ to generate comprehensive descriptions of critical moments in these episodes, as shown in Fig.~\ref{framework}(b). Specifically, we employ GPT-4o as our primary vision-language model due to its superior performance in comprehensive scene understanding and detailed description generation. Our approach particularly focuses on interactions with critical objects, i.e., traffic participants whose interactions with the AV agent lead to unsafe situations. Given a task description and instruction $l$, the VLM processes the visual observations to generate comprehensive narrative descriptions for each safety-critical event:
\begin{equation}\label{eq3}
    d_i = f_{\text{VLM}}(\l, \mathcal{O}^{\leq k}_i),
\end{equation}
where $\mathcal{O}^{\leq k}_i$ contains the sequence of key frames with length less than or equal to $k$ that capture the safety-critical event and its context. The task description $l$ guides the VLM to focus on the type and nature of safety violations, spatial relationships between vehicles, and causal factors contributing to unsafe situations.

The generated description $d_i$ provides a structured characterization of the event, including: (1) the type and nature of the safety violation (e.g., rear-end collision, side-swipe), (2) the relative positioning and behavior of critical objects, (3) the AV's response patterns, and (4) the contextual factors contributing to the incident. For instance, as illustrated in Fig.~\ref{framework}(b), in a failed lane change scenario, the VLM might generate: ``The incident involves a rear-end collision during a lane change maneuver. The red vehicle (AV) is driving forward attempting to merge, while the collided vehicle is positioned behind in the target lane. The collision occurred due to insufficient acceleration and improper speed matching during the lane transition."

\subsubsection{Behavioral Pattern Analysis and Recommenddation}

Building upon the VLM's descriptions of critical events and identified critical objects, we leverage GPT-4o to perform systematic behavioral pattern analysis. Our framework innovatively adopts a batch-wise analysis strategy by accumulating $N$ safety-critical events before initiating the analysis phase, as illustrated in Fig.~\ref{framework}(b). This approach enables us to process a sequence of recent event descriptions, denoted as $\mathcal{D}_{1:N} = \{d_i\}_{i=1}^N$. The batch-wise strategy offers two key advantages: (1) it enables the identification of recurring behavioral patterns that might be overlooked in isolated incident analysis, and (2) it optimizes computational efficiency by reducing the frequency of model queries.

To facilitate structured analysis and ensure the quality and consistency of generated responses, we design a comprehensive prompt template that consists of three key components. First, we establish the analysis context by positioning the model as an expert in autonomous vehicle performance analysis and scenario generation, providing essential background for understanding safety-critical scenarios in autonomous driving. Second, we provide a systematic presentation of safety-critical events by incorporating the detailed descriptions from visual observation understanding: ``Scenario 1: [VLM Response 1] ... Scenario $N$: [VLM Response $N$]''. This organization ensures that all relevant spatial relationships, vehicle behaviors, and safety violations are clearly presented. Third, we guide the analysis process with specific instructions for behavioral pattern identification and improvement recommendation generation, emphasizing the importance of both immediate safety concerns and long-term behavioral improvements.

The behavioral analysis process can be formalized as:
\begin{equation}\label{eq4}
    \mathcal{I} = g_{\text{GPT-4o}}(P(\mathcal{D}_{1:N})),
\end{equation}
where $P(\cdot)$ represents our carefully designed prompt template that structures the input descriptions $\mathcal{D}_{1:N}$, and $g_{\text{GPT-4o}}(\cdot)$ denotes the analysis process that generates insights $\mathcal{I}$.

To ensure comprehensive and systematic analysis, we specify a structured response format that guides the reasoning process through four key components. The response format begins with \textit{Situation} assessment, which provides a comprehensive examination of common patterns across the $N$ safety-critical events. This component identifies recurring scenarios and behavioral tendencies in the agent's interactions with other traffic participants. This is followed by \textit{Reasoning}, where in-depth causal analysis examines the underlying factors contributing to safety violations and behavioral patterns in the agent's decision-making process. Building upon this understanding, the format requires a specific \textit{Recommendation} generation to produce suggestions for improving agent performance. Lastly, the \textit{CriticalArea} component distills these recommendations to identify key regions that require focused training attention. These identified regions provide actionable guidance for generating personalized safety-critical scenarios in the subsequent curriculum adaptation process.

This structured format ensures that the analysis systematically progresses from pattern recognition through causal understanding to concrete improvement strategies. For instance, given a series of safety-critical events, the analysis might generate: ``Situation: The AV agent has been involved in multiple rear-end collisions during lane merge scenarios, particularly when other vehicles merge into its lane. Reasoning: The agent consistently fails to maintain safe following distances and demonstrates delayed responses to merging vehicles, suggesting limitations in its predictive capabilities. Recommendation: The agent should be trained to proactively adjust its speed and position when detecting potential merging intentions from neighboring vehicles, with emphasis on early recognition of lane change indicators. CriticalArea: The analysis indicates particular vulnerability in handling vehicles merging from adjacent lanes, specifically in the `Left/Right' spatial regions relative to the ego vehicle.''

\subsection{Personalized Curriculum Adaptation via Scenario Generation}\label{sec4.3}

Building upon the comprehensive behavioral analysis and recommendations from GPT-4o, a critical challenge emerges: \textit{how to effectively translate these high-level recommendations into concrete, executable scenarios that target the agent's identified weaknesses.} To address this challenge, we propose a curriculum adaptation mechanism that dynamically generates personalized training scenarios based on the generated insights $\mathcal{I}$.

To transform GPT-4o's recommendations into executable scenarios, we formulate it as a conditional trajectory generation problem. Let $X = (M, S^{\text{AV}}_{1:t}, S^{\text{BV}}_{1:t})$ denote the historical information up to timestep $t$, including the HD map $M$ and past states of both AV agent and nearby background vehicles (BVs). Let $Y^{\text{AV}} = S^{\text{AV}}_{t:T}$ and $Y^{\text{BV}} = S^{\text{BV}}_{t:T}$ represent their future trajectories respectively. Given the behavioral insights and historical information, the distribution of future trajectories in the generated personalized scenario can be expressed as:
\begin{equation}
    \mathbb{P}(Y^{\text{AV}}, Y^{\text{BV}}|\mathcal{I}, X).
\end{equation}
To generate scenarios that effectively align with current insight $\mathcal{I}$, we seek to maximize this probability by finding an optimal BV trajectory:
\begin{equation}\label{eq6}
Y^{\text{BV}*} = \arg\max_{Y^{\text{BV}}} \sum_{Y^{\text{AV}} \sim \mathcal{Y}(\pi)} \mathbb{P}(Y^{\text{AV}}, Y^{\text{BV}}|\mathcal{I}, X),
\end{equation}
where $\mathcal{Y}(\pi)$ denotes the set of possible trajectories generated by the current AV policy $\pi$.

However, directly optimizing the objective in Eq.~\eqref{eq6} is intractable due to the complex interdependencies between the AV's trajectory and surrounding BVs. Inspired by recent works \citep{sun2022m2i,zhang2023cat}, we observe that in safety-critical scenarios, the danger often originates from nearby BVs' aggressive or unexpected behaviors, which require appropriate responsive actions from the AV agent to maintain safety. For instance, when a BV suddenly cuts in at a close distance, the AV agent must quickly recognize the danger and execute appropriate defensive maneuvers such as emergency braking or evasive steering. This observation suggests a predominantly unidirectional dependency in the interaction pattern, i.e., the AV agent needs to adapt its behavior based on the observed and predicted trajectories of surrounding vehicles to ensure safety.

Based on this observation, we can decompose the complex joint distribution into more manageable components through factorization and Bayes' rule:
\begin{equation}
    \mathbb{P}(Y^{\text{AV}}, Y^{\text{BV}}|\mathcal{I}, X) \propto 
    \mathbb{P}(Y^{\text{BV}}|X) \cdot
    \mathbb{P}(Y^{\text{AV}}|Y^{\text{BV}}, X) \cdot
    \mathbb{P}(\mathcal{I}|Y^{\text{AV}}, Y^{\text{BV}}).
\end{equation}
As illustrated in Fig.~\ref{framework}(c), this factorization decomposes the scenario generation process into three interpretable components: $\mathbb{P}(Y^{\text{BV}}|X)$ represents the marginal distribution of BV's trajectories, $\mathbb{P}(Y^{\text{AV}}|Y^{\text{BV}}, X)$ captures the conditional distribution of the AV's responses to these trajectories, and $\mathbb{P}(\mathcal{I}|Y^{\text{AV}}, Y^{\text{BV}})$ serves as a probability score measuring how well the generated scenario aligns with the identified behavioral insights. By substituting this factorized form into Eq.~\eqref{eq6}, we obtain a more tractable optimization objective:
\begin{equation}\label{eq8}
\begin{split}
Y^{\text{BV}*} &= \arg\max_{Y^{\text{BV}}} \sum_{Y^{\text{AV}} \sim \mathcal{Y}(\pi)} \mathbb{P}(Y^{\text{AV}}, Y^{\text{BV}}|\mathcal{I}, X)\\
&= \arg\max_{Y^{\text{BV}}} \mathbb{P}(Y^{\text{BV}}|X)
    \sum_{Y^{\text{AV}} \sim \mathcal{Y}(\pi)}
    \mathbb{P}(Y^{\text{AV}}|Y^{\text{BV}}, X) \cdot
    \mathbb{P}(\mathcal{I}|Y^{\text{AV}}, Y^{\text{BV}}).\\
\end{split}
\end{equation}

The first term $\mathbb{P}(Y^{\text{BV}}|X)$ can be regarded as a prior distribution of BVs' trajectories conditioned on historical information. This distribution needs to capture both the physical feasibility of vehicle motion and the diverse patterns of human driving behaviors. While recent advances in trajectory prediction have demonstrated significant progress in modeling such distributions \citep{gu2021densetnt,sheng2024ego,sheng2024kinematics}, developing a new trajectory prediction model from scratch would be potentially redundant. Therefore, we leverage existing well-validated models to ensure reliable and efficient scenario generation. Specifically, we adopt the pre-trained DenseTNT model \citep{gu2021densetnt} as our trajectory prior generator. DenseTNT is particularly suitable for our task due to its strong capability in modeling multi-modal behaviors while ensuring physical feasibility through vectorized HD map encoding. To approximate the trajectory prior distribution $\mathbb{P}(Y^{\text{BV}}|X)$, we utilize DenseTNT to generate a diverse set of $K$ candidate trajectories $\{Y^{\text{BV}}_k\}_{k=1}^K$ for each BV, along with their corresponding likelihood scores $\{p^{\text{BV}}_k\}_{k=1}^K$. These trajectory-probability pairs form an empirical approximation of the prior distribution.

The second term $\mathbb{P}(Y^{\text{AV}}|Y^{\text{BV}}, X)$ represents the conditional distribution of AV's responses given the trajectories of BVs. Unlike the BV trajectories which can be sampled from a learned prior, modeling this conditional distribution is challenging as the AV agent's capacities continuously evolve throughout the training process. Directly sampling from pre-learned distributions would introduce significant behavioral bias and fail to capture the agent's current performance characteristics. To address this challenge, we propose to approximate this conditional distribution using the agent's most recent behavioral patterns. We maintain a buffer of the agent's latest rollouts $\mathcal{B} = \{(s_1^i, a_1^i, p_1^i, ..., s_T^i, a_T^i, p_T^i)\}_{i=1}^L$, where each rollout consists of a sequence of state-action-probablity pairs. For each action in these rollouts, we compute its cumulative probability that considers both historical context and the current policy: $p_t^i = p_{t-1}^i \cdot \pi(a_t^i|s_t^i),$ where $\pi(a_t^i|s_t^i)$ is the probability of taking action $a_t^i$ under the current policy given state $s_t^i$. This rolling buffer mechanism ensures that our approximated response distribution remains synchronized with the agent's evolving capabilities.

The third term $\mathbb{P}(\mathcal{I}|Y^{\text{AV}}, Y^{\text{BV}})$ evaluates how well the generated scenario aligns with the behavioral insights obtained from our GPT-4o analysis framework. This probability score assesses whether the interaction between AV and BV trajectories effectively creates situations that target the identified behavioral weaknesses. For each candidate AV trajectory $Y^{\text{AV}}_j$ and BV trajectory $Y^{\text{BV}}_i$, we evaluate their alignment with the identified critical areas $\mathcal{I}$ through a probability function:
\begin{equation}\label{eq9}
    P^{\mathcal{I}}_{i,j} = \begin{cases}
        \lambda^t & \text{if }\mathbb{I}(t, Y^{\text{AV}}_j, Y^{\text{BV}}_i, \mathcal{I}) = \text{True} \\
        0 & \text{otherwise}
    \end{cases}
\end{equation}
where $\lambda \in (0,1)$ is a decay factor reflecting increasing uncertainty over time, and $\mathbb{I}(\cdot)$ is a binary function that determines the spatial alignment between potential collisions and identified critical areas in $\mathcal{I}$. During trajectory evaluation, this function first detects collision events between AV and BV trajectories at timestep $t$, then assesses whether the BV's relative position (Front, Behind, Left, or Right) with respect to the AV corresponds to any critical area specified in $\mathcal{I}$, returning True only when both conditions are satisfied.

Building upon the three factorized probability components discussed above, we propose an efficient scoring mechanism to evaluate candidate BV trajectories based on their alignment with the current behavioral insights $\mathcal{I}$. The scoring function is formulated as:
\begin{equation}\label{eq10}
\text{Score}(Y^{\text{BV}}_i) =  p^{\text{BV}}_i \sum_{j=1}^L \sum_{t=1}^T p_t^j \cdot P_{i,j}^{\mathcal{I}},
\end{equation}
where $p^{\text{BV}}_i$ represents the prior probability of the candidate BV trajectory obtained from DenseTNT, and $p_t^j$ denotes the cumulative probability of the sampled AV response from the behavior buffer $\mathcal{B}$. 
The optimal BV trajectory $Y^{\text{BV}*}$ for scenario generation is then selected through:
\begin{equation}\label{eq11}
    Y^{\text{BV}*} = \arg\max_{Y^{\text{BV}}_i} \text{Score}(Y^{\text{BV}}_i).
\end{equation}
This formulation allows us to generate personalized safety-critical scenarios that specifically target the agent's identified behavioral weaknesses. By integrating realistic trajectory priors from DenseTNT and explicitly evaluating alignment with the safety-critical patterns identified through our GPT-4o analysis framework, our approach ensures the generation of meaningful and challenging learning experiences. These generated scenarios serve as effective curriculum elements for improving the agent's safety performance. The detailed steps of this scenario generation process are summarized in Algorithm \ref{alg:generate}.

\subsection{Dynamic Curriculum Scheduling for RL}\label{sec4.4}

The effectiveness of our CurricuVLM framework critically depends on how we integrate the generated safety-critical scenarios into the training process. We propose a dynamic scheduling mechanism that adaptively adjusts the composition of training scenarios based on the agent's evolving capabilities. This mechanism needs to address two challenges: (1) establishing and maintaining basic driving competencies through regular scenarios, and (2) developing robust safety-critical handling capabilities through the personalized scenarios generated by our VLM-based analysis framework.

Building upon Definition \ref{def2}, we formalize our curriculum $\mathcal{C}$ as a sequence of training episodes $\{\xi_1, \xi_2, ..., \xi_N\}$, where each episode can be either a regular driving scenario or a personalized safety-critical scenario. 
As illustrated in Fig.~\ref{framework}(d), for each episode, we implement a probability-based selection mechanism to determine the type of scenario to be used. The probability of incorporating a safety-critical scenario for episode $\xi_{i+1}$ is determined by an adaptive triggering mechanism:
\begin{equation}
    p_t = \min(\kappa \cdot \frac{t}{T} \cdot p_{\text{max}}, p_{\text{max}}),
\end{equation}
where $t$ denotes the current training step, $T$ represents the total training steps, $\kappa$ controls the rate of difficulty progression, and $p_{\text{max}}$ sets an upper bound on the triggering probability. This formulation ensures that the exposure to safety-critical scenarios increases steadily with the agent's development while remaining bounded by $p_{\text{max}}$, which creates a progressive learning curriculum consistent with Definition \ref{def2}. When triggered, a new safety-critical scenario $\xi_{i+1}$ is generated using the method detailed in Section \ref{sec4.3} to incorporate the latest behavioral insights $\mathcal{I}$.

The dynamic nature of this scheduling mechanism enables the curriculum to adapt to the agent's learning progress. During early training phases when $p_t$ is low, the agent primarily encounters regular driving scenarios, allowing it to develop fundamental driving skills. As training progresses and $p_t$ increases, the curriculum gradually incorporates more safety-critical scenarios that specifically target the behavioral weaknesses identified through our GPT-4o analysis. 
As proven in \ref{theory}, our scheduling mechanism theoretically guarantees the convergence of learning by ensuring sufficient exposure to both regular and safety-critical scenarios throughout the training process.
The overall algorithmic procedure of our CurricuVLM framework is detailed in Algorithm \ref{alg:curricuVLM}.

\section{Experiments}\label{sec5}

\begin{figure}[!t]
  \centerline{\includegraphics[width=0.99\textwidth]{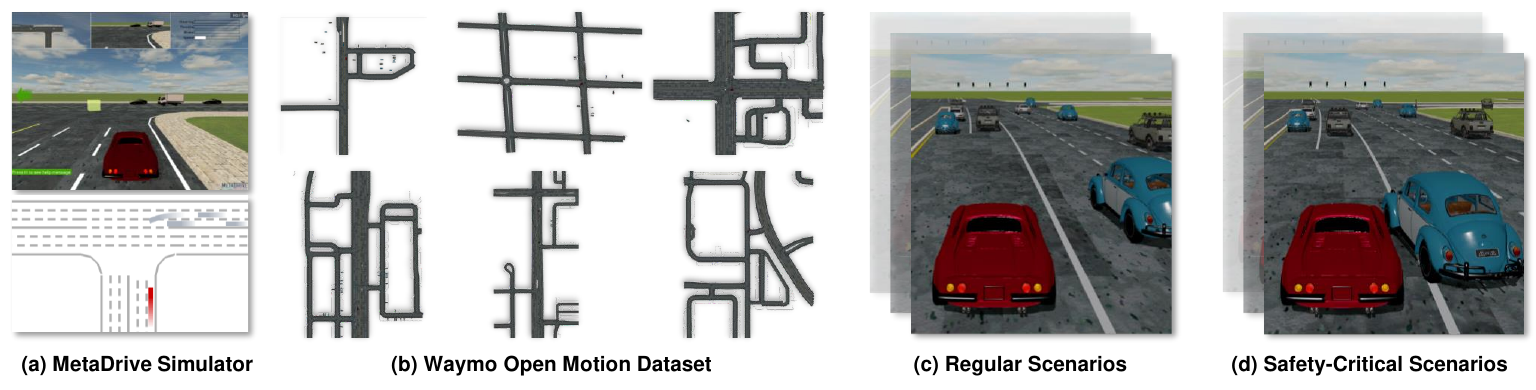}}
  \caption{Experiment environment and evaluation scenarios. (a) MetaDrive simulator with 3D render and top-down view; (b) Representative HD map layouts from Waymo Open Motion Dataset; (c) Regular scenarios with original background vehicle trajectories; (d) Generated safety-critical scenarios with challenging vehicle interactions.}
  \label{exp-env}
\end{figure}

\subsection{Experimental Setups}
\subsubsection{Experiment Environment}
We conduct our experiments using the MetaDrive simulator \citep{li2022metadrive}, an open-source lightweight autonomous driving platform that provides flexible environment configurations and efficient scenario simulation, as shown in Fig.~\ref{exp-env}(a). For training and evaluation, we leverage the Waymo Open Motion Dataset \citep{ettinger2021large}, which provides diverse real-world driving scenarios collected across various traffic conditions and road layouts (Fig.~\ref{exp-env}(b)). Following the existing work \citep{zhang2023cat}, we use 500 representative scenarios from the dataset and create a split of 400 for training and 100 for testing. Each scenario contains comprehensive trajectory information of the ego vehicle and surrounding traffic participants, along with detailed HD map annotations.

\subsubsection{RL Setting}

The observation space consists of three key components that provide comprehensive information about the driving environment. First, the ego state includes the vehicle's dynamic properties such as speed, steering angle, throttle/brake position, and yaw rate, which characterize the vehicle's current motion state. Second, navigation information is represented by the relative distance and heading to the next two waypoints along the planned route, providing guidance for RL agent's future actions. Third, LiDAR point cloud data captures the surrounding environment, enabling the agent to perceive and react to obstacles and other traffic participants.

The agent controls the vehicle through a continuous action space $\mathcal{A} \in [-1, 1]^2$, consisting of steering and throttle/brake commands. The steering action $a_{\text{steer}} \in [-1, 1]$ determines the steering angle, where negative values indicate left turns and positive values indicate right turns. The throttle/brake action $a_{\text{acc}} \in [-1, 1]$ controls longitudinal motion, with positive values applying throttle for acceleration and negative values engaging the brake for deceleration.

We adopt the default reward function from MetaDrive \citep{li2022metadrive}, which combines dense guidance for basic driving behaviors with collision penalties and sparse rewards for episode outcomes:
\begin{equation}
    R(s_t, a_t) = R_{\text{drive}}(s_t, a_t) + R_{\text{collision}} + R_{\text{done}},
\end{equation}
where $R_{\text{drive}}(s_t, a_t)$ denotes the dense reward for basic driving behaviors, $R_{\text{collision}}$ represents the penalty for collisions, and $R_{\text{done}}$ provides rewards based on episode outcomes.

The driving reward term $R_{\text{driving}}$ encourages the AV agent to move forward while maintaining its position at the center of the lane. It is formulated as:
\begin{equation}
    R_{\text{drive}}(s_t, a_t) = \begin{cases} 
    0 & \text{if the episode is done or collision occurs}, \\
    \alpha \cdot d_{\text{lon}} \cdot \text{clip}\left(1 - \frac{2|d_{\text{lat}}|}{d_{\text{max}}}, 0, 1\right) & \text{otherwise}
    \end{cases}
\end{equation}
where $\alpha$ is the driving reward coefficient, $d_{\text{lon}}$ represents the longitudinal distance traveled between current and last timesteps. The $\text{clip}(\cdot)$ function calculates a lateral position factor, where $d_{\text{lat}}$ denotes the lateral distance from the lane center, and $d_{\text{max}}$ is the maximum lateral distance threshold. The lateral factor decreases linearly with the deviation from the lane center and is clipped between 0 and 1, thereby encouraging the vehicle to maintain a central lane position.
The collision reward $R_{\text{collision}}$ provides immediate feedback when unsafe behaviors occur, applying a penalty of -1 whenever a collision is detected while maintaining 0 reward during regular driving. The done reward $R_{\text{done}}$ only triggers when an episode ends, providing a positive reward +10 when the AV successfully reaches its destination, or a negative penalty -10 if the episode terminates due to the agent entering a non-drivable area.

Given that our personalized safety-critical scenarios continuously evolve based on the VLM analysis, we require an online RL algorithm that can efficiently adapt to the changing distribution of training scenarios. Therefore, we adopt the TD3 \citep{fujimoto2018addressing} as our default RL algorithm due to its efficient off-policy learning capability and actor-critic architecture, which are particularly suitable for continuous control tasks in autonomous driving. TD3's ability to learn from off-policy data enables efficient utilization of experiences collected from both regular driving and safety-critical scenarios, facilitating robust learning across diverse driving situations.  
While TD3 serves as the primary algorithm for evaluation, our CurricuVLM framework is designed to be algorithm-agnostic. Our subsequent experimental analysis in Section \ref{sec5.6} will investigate the framework's compatibility with other popular RL algorithms such as PPO and SAC, examining the potential of our CurricuVLM in enhancing various RL algorithms.

\subsubsection{Evaluation Strategy and Metrics}

Our evaluation consists of two phases to comprehensively assess the effectiveness of different methods:
\begin{enumerate}[label=\arabic*)]
 \item \textbf{Regular Scenario Testing}: In this phase, we evaluate each method by having the AV agent navigate through test scenarios while BVs follow their original trajectories from the dataset, as illustrated in Fig.~\ref{exp-env}(c). This provides a baseline assessment of basic driving capabilities under regular conditions.
 \item \textbf{Safety-Critical Scenario Testing}: We transform the test scenarios into safety-critical versions using our scenario generation method in Section \ref{sec4.3}. Fig.~\ref{exp-env}(d) shows examples of such transformed scenarios, where BVs are positioned to create challenging situations. The methods are then evaluated under these conditions to assess their robustness in handling safety-critical situations.
\end{enumerate}

This dual-phase evaluation approach enables us to systematically assess both the fundamental driving competencies and safety-critical handling capabilities of different autonomous driving methods. To quantitatively measure performance in both phases, we employ multiple metrics focusing on safety, navigation performance, and learning capabilities. Safety performance is primarily measured through the \textit{crash rate}, which quantifies the percentage of episodes ending in collisions. The success of navigation tasks is assessed using the \textit{road completion rate} and \textit{total distance} traveled, where higher values indicate better route-following behavior. We also report \textit{average speed} to evaluate the agent's ability to balance efficiency and safety.

To specifically evaluate the learning effectiveness, we introduce two specialized metrics: the \textit{failure-to-success rate} and \textit{success-to-success rate}. The failure-to-success rate measures the percentage of previously failed scenarios where the agent achieves success after training, directly quantifying the agent's improvement in handling challenging situations. The success-to-success rate evaluates the agent's performance consistency by measuring the percentage of scenarios where the agent maintains successful performance, indicating the stability and reliability of the learned policy. Additionally, we use \textit{episode reward} as a comprehensive metric that captures the overall performance of the driving policy.

\subsection{Baselines}
We compare our method against several state-of-the-art baselines across four categories:

\begin{itemize}
    \item \textbf{Traditional RL:} We include SAC \citep{Haarnoja2018Soft} which employs maximum entropy principles to enhance exploration, PPO \citep{Schulman2017Proximal} which uses trust region optimization for stable policy updates, and TD3 \citep{fujimoto2018addressing} which addresses overestimation bias in value function estimation.
    
    \item \textbf{Safe RL:} This category comprises SAC-PID \citep{stooke2020responsive} which incorporates PID controllers into the SAC framework, TD3-Lag \citep{ray2019benchmarking} which employs Lagrangian methods to handle safety constraints, and TD3-PID \citep{stooke2020responsive} which combines TD3 with PID control for enhanced safety guarantees.
    
    \item \textbf{Imitation Learning:} We evaluate against BC \citep{gleave2022imitation} which directly learns from expert demonstrations, GAIL \citep{ho2016generative} which employs adversarial training to match expert behavior, AIRL \citep{fu2017learning} which learns a robust reward function from demonstrations, and SQIL \citep{reddy2019sqil} which incorporates Q-learning into the imitation framework.
    
    \item \textbf{Closed-loop Curriculum Learning:} We include CAT \citep{zhang2023cat} which dynamically generates challenging driving scenarios during training. We also compare with CLIC \citep{niu2024continual} which employs a difficulty predictor to estimate collision probabilities and reweights the scenario sampling distribution accordingly.
\end{itemize}

\begin{table}[!t]
\begin{small}
  \caption{Performance comparison with baselines in the \textit{regular test scenarios}. Mean and standard deviation over 3 seeds. The best results are marked in \textbf{bold}.}
  \label{tab1}
  \centering
  \begin{center}
  \renewcommand{\arraystretch}{1.5}
  \resizebox{\textwidth}{!}{
  \begin{tabular}{cccccccc}
  \toprule
  % \multirow{2}{*}{Category} & 
  \multirow{2}{*}{Model} &  \multirow{2}{*}{\shortstack{Episode\\Reward}} \multirow{2}{*}{$\uparrow$} & \multirow{2}{*}{\shortstack{Road\\Completion (\%)}} \multirow{2}{*}{$\uparrow$} & \multirow{2}{*}{\shortstack{Total\\Distance}} \multirow{2}{*}{$\uparrow$} & \multirow{2}{*}{\shortstack{Crash\\Rate (\%)}} \multirow{2}{*}{$\downarrow$} & \multirow{2}{*}{\shortstack{Average\\Speed}} \multirow{2}{*}{$\uparrow$} & \multirow{2}{*}{\shortstack{Failure-to-\\Success Rate (\%)}} \multirow{2}{*}{$\uparrow$} & \multirow{2}{*}{\shortstack{Success-to-\\Success Rate (\%)}} \multirow{2}{*}{$\uparrow$} \\
  & & & & & & & \\
  \midrule 
  
  % \multirow{3}{*}{\shortstack{Traditional\\RL}} & 
  SAC & 43.3 {\tiny $\pm$ 1.61} & 67.9 {\tiny $\pm$ 1.14} & 44.9 {\tiny $\pm$ 1.07} & 18.3 {\tiny $\pm$ 0.92} & 9.35 {\tiny $\pm$ 0.06} & 38.7 {\tiny $\pm$ 6.32} & 68.9 {\tiny $\pm$ 13.7}   \\
  PPO & 41.1 {\tiny $\pm$ 2.48} & 65.9 {\tiny $\pm$ 2.51} & 42.9 {\tiny $\pm$ 2.01} & 22.6 {\tiny $\pm$ 3.75} & \textbf{10.1} {\tiny $\pm$ 0.24} & 36.0 {\tiny $\pm$ 0.54} & 64.3 {\tiny $\pm$ 7.14} \\
  TD3 & 49.3 {\tiny $\pm$ 0.83} & 72.9 {\tiny $\pm$ 0.44} & 47.9 {\tiny $\pm$ 1.08} & 18.9 {\tiny $\pm$ 1.72} & 8.14  {\tiny $\pm$ 0.80} & 54.1 {\tiny $\pm$ 0.14} & 67.1 {\tiny $\pm$ 3.48} \\
  \midrule 
  
  % \multirow{3}{*}{Safe RL} & 
  SAC-PID & 7.57 {\tiny $\pm$ 3.20} & 24.8 {\tiny $\pm$ 10.3} & 12.9 {\tiny $\pm$ 2.91} & 21.3 {\tiny $\pm$ 1.35} & 5.19  {\tiny $\pm$ 0.57} & 1.45 {\tiny $\pm$ 1.40} & 10.0 {\tiny $\pm$ 14.1} \\
  TD3-Lag & 39.6 {\tiny $\pm$ 0.0} & 46.9 {\tiny $\pm$ 0.00} & 38.9 {\tiny $\pm$ 0.00} & 23.0 {\tiny $\pm$ 0.00} & 5.50 {\tiny $\pm$ 0.00} & 29.0 {\tiny $\pm$ 0.43} & 72.2 {\tiny $\pm$ 20.8} \\
  TD3-PID & -1.00 {\tiny $\pm$ 0.91} & 6.48 {\tiny $\pm$ 0.44} & 4.55 {\tiny $\pm$ 0.37} & 22.8 {\tiny $\pm$ 1.77} & 4.17 {\tiny $\pm$ 1.08} & 0.37 {\tiny $\pm$ 0.52} & 7.41 {\tiny $\pm$ 10.5}  \\
  \midrule
  
  % \multirow{4}{*}{\shortstack{Imitation\\Learning}} & 
   BC    & 29.7 {\tiny $\pm$ 2.52} & 58.6 {\tiny $\pm$ 2.09} & 33.4 {\tiny $\pm$ 1.80} & 20.6 {\tiny $\pm$ 1.18} & 7.42  {\tiny $\pm$ 0.28} & 34.9 {\tiny $\pm$ 1.32} & 41.2 {\tiny $\pm$ 11.6} \\
   GAIL  & 32.5 {\tiny $\pm$ 1.99} & 57.0 {\tiny $\pm$ 1.50} & 35.1 {\tiny $\pm$ 1.82} & 23.4 {\tiny $\pm$ 1.94} & 7.84  {\tiny $\pm$ 0.74} & 20.5 {\tiny $\pm$ 11.8} & 46.7 {\tiny $\pm$ 41.1}  \\
   AIRL  & 27.0 {\tiny $\pm$ 1.94} & 54.8 {\tiny $\pm$ 3.13} & 32.0 {\tiny $\pm$ 1.63} & 20.1 {\tiny $\pm$ 3.06} & 9.80  {\tiny $\pm$ 0.76} & 19.7 {\tiny $\pm$ 9.71} & 16.7 {\tiny $\pm$ 23.6}  \\
   SQIL  & 20.3 {\tiny $\pm$ 0.21} & 44.7 {\tiny $\pm$ 0.78} & 25.1 {\tiny $\pm$ 0.36} & 26.9 {\tiny $\pm$ 1.30} & 6.50 {\tiny $\pm$ 0.38} & 6.98 {\tiny $\pm$ 4.02} & 45.5 {\tiny $\pm$ 15.1} \\
  \midrule 

  % \multirow{2}{*}{\shortstack{Close-Loop}} & 
   CAT  & 49.1 {\tiny $\pm$ 3.89} & 73.0 {\tiny $\pm$ 3.94} & 48.5 {\tiny $\pm$ 3.39} & 15.5 {\tiny $\pm$ 1.72} & 8.45 {\tiny $\pm$ 1.23} & \textbf{55.7} {\tiny $\pm$ 4.55} & 68.8 {\tiny $\pm$ 9.74}  \\
   CLIC  & 43.5 {\tiny $\pm$ 0.49} & 68.3 {\tiny $\pm$ 0.60} & 45.1 {\tiny $\pm$ 0.71} & \textbf{15.1} {\tiny $\pm$ 1.39} & 9.31 {\tiny $\pm$ 0.28} & 39.1 {\tiny $\pm$ 3.45} & 48.3 {\tiny $\pm$ 19.3}  \\
  \rowcolor{gray!15}
   CurricuVLM (ours) & \textbf{52.3} {\tiny $\pm$ 2.28} & \textbf{76.2} {\tiny $\pm$ 2.28} & \textbf{51.0} {\tiny $\pm$ 1.65} & 16.0 {\tiny $\pm$ 0.80} & 9.52 {\tiny $\pm$ 0.14} & 51.3 {\tiny $\pm$ 6.26} & \textbf{75.0} {\tiny $\pm$ 2.91}   \\
  \bottomrule
  \end{tabular}
  }
  \end{center}
\end{small}
\end{table}

\begin{table}[!t]
\begin{small}
  \caption{Performance comparison with baselines in the \textit{safety-critical test scenarios}. Mean and standard deviation over 3 seeds. The best results are marked in \textbf{bold}.}
  \label{tab2}
  \centering
  \begin{center}
  \renewcommand{\arraystretch}{1.5}
  \resizebox{\textwidth}{!}{
  \begin{tabular}{cccccccc}
  \toprule
  % \multirow{2}{*}{Category} & 
  \multirow{2}{*}{Model} &  \multirow{2}{*}{\shortstack{Episode\\Reward}} \multirow{2}{*}{$\uparrow$} & \multirow{2}{*}{\shortstack{Road\\Completion (\%)}} \multirow{2}{*}{$\uparrow$} & \multirow{2}{*}{\shortstack{Total\\Distance}} \multirow{2}{*}{$\uparrow$} & \multirow{2}{*}{\shortstack{Crash\\Rate (\%)}} \multirow{2}{*}{$\downarrow$} & \multirow{2}{*}{\shortstack{Average\\Speed}} \multirow{2}{*}{$\uparrow$} & \multirow{2}{*}{\shortstack{Failure-to-\\Success Rate (\%)}} \multirow{2}{*}{$\uparrow$} & \multirow{2}{*}{\shortstack{Success-to-\\Success Rate (\%)}} \multirow{2}{*}{$\uparrow$} \\
  & & & & & & & \\
  \midrule
  
  % \multirow{3}{*}{\shortstack{Traditional\\RL}} & 
  SAC & 38.4 {\tiny $\pm$ 1.97} & 63.2 {\tiny $\pm$ 1.21} & 40.9 {\tiny $\pm$ 1.34} & 30.5 {\tiny $\pm$ 2.33} & 9.25 {\tiny $\pm$ 0.07}   & 30.4 {\tiny $\pm$ 7.00} & 56.9 {\tiny $\pm$ 15.1}   \\
  PPO & 38.4 {\tiny $\pm$ 0.86} & 62.7 {\tiny $\pm$ 1.05} & 40.0 {\tiny $\pm$ 0.70} & 32.0 {\tiny $\pm$ 2.02} & \textbf{9.94} {\tiny $\pm$ 0.30}  & 26.7 {\tiny $\pm$ 0.89} & 41.7 {\tiny $\pm$ 8.33} \\
  TD3 & 42.4 {\tiny $\pm$ 1.01} & 65.2 {\tiny $\pm$ 1.40} & 42.6 {\tiny $\pm$ 1.26} & 39.7 {\tiny $\pm$ 1.04} & 8.02 {\tiny $\pm$ 0.77}   & 28.6 {\tiny $\pm$ 2.79} & 64.3 {\tiny $\pm$ 21.4} \\
  \midrule 
  
  % \multirow{3}{*}{Safe RL} & 
  SAC-PID & 7.46 {\tiny $\pm$ 3.05} & 19.0 {\tiny $\pm$ 4.19} & 12.7 {\tiny $\pm$ 2.82} & 30.4 {\tiny $\pm$ 3.27} & 5.22 {\tiny $\pm$ 0.54} & 2.13 {\tiny $\pm$ 3.01} & 16.7 {\tiny $\pm$ 23.6} \\
  TD3-Lag & 29.1 {\tiny $\pm$ 0.00} & 39.0 {\tiny $\pm$ 0.00} & 30.7 {\tiny $\pm$ 0.00} & 51.0 {\tiny $\pm$ 0.00} & 5.70 {\tiny $\pm$ 0.00}  & 8.42 {\tiny $\pm$ 0.48} & 66.7 {\tiny $\pm$ 47.1} \\
  TD3-PID & -0.68 {\tiny $\pm$ 1.00} & 5.50 {\tiny $\pm$ 0.90} & 4.58 {\tiny $\pm$ 0.35} & 29.4 {\tiny $\pm$ 5.83} & 4.18 {\tiny $\pm$ 1.06}  & 1.39 {\tiny $\pm$ 1.32} & 13.3 {\tiny $\pm$ 18.9}  \\
  \midrule
  
  % \multirow{4}{*}{\shortstack{Imitation\\Learning}} & 
   BC    &  27.4 {\tiny $\pm$ 2.69} & 53.2 {\tiny $\pm$ 1.81} & 31.0 {\tiny $\pm$ 2.17} & 33.7 {\tiny $\pm$ 0.27} & 7.30 {\tiny $\pm$ 0.26}    & 26.5 {\tiny $\pm$ 3.43} & 32.3 {\tiny $\pm$ 15.5} \\
   GAIL  &  29.8 {\tiny $\pm$ 1.07} & 52.3 {\tiny $\pm$ 1.42} & 32.7 {\tiny $\pm$ 1.05} & 33.7 {\tiny $\pm$ 1.78} & 7.77 {\tiny $\pm$ 0.73}    & 16.7 {\tiny $\pm$ 7.16} & 4.76 {\tiny $\pm$ 6.73}  \\
   AIRL  &  26.3 {\tiny $\pm$ 1.84} & 52.8 {\tiny $\pm$ 3.16} & 31.2 {\tiny $\pm$ 1.50} & 27.8 {\tiny $\pm$ 3.51} & 9.74 {\tiny $\pm$ 0.72}   & 18.8 {\tiny $\pm$ 9.37}  & 33.3 {\tiny $\pm$ 47.1}  \\
   SQIL  &  18.9 {\tiny $\pm$ 0.56} & 41.4 {\tiny $\pm$ 1.01} & 23.5 {\tiny $\pm$ 0.52} & 38.5 {\tiny $\pm$ 1.67} & 6.54 {\tiny $\pm$ 0.34} & 7.12 {\tiny $\pm$ 2.62} & 41.7 {\tiny $\pm$ 18.0} \\
  \midrule 

  % \multirow{2}{*}{\shortstack{Close-Loop}} & 
   CAT  &  42.5 {\tiny $\pm$ 3.95} & 66.6 {\tiny $\pm$ 4.37} & 43.4 {\tiny $\pm$ 3.48} & 32.1 {\tiny $\pm$ 2.08} & 8.36 {\tiny $\pm$ 1.17}  & 35.2 {\tiny $\pm$ 3.44} & 67.5 {\tiny $\pm$ 7.50}  \\
   CLIC  & 39.3 {\tiny $\pm$ 0.72} & 64.3 {\tiny $\pm$ 0.40} & 41.6 {\tiny $\pm$ 0.78} & 26.2 {\tiny $\pm$ 1.17} & 9.21 {\tiny $\pm$ 0.26} & 34.7 {\tiny $\pm$ 2.67} & 61.9 {\tiny $\pm$ 26.9}  \\
  \rowcolor{gray!15}
  CurricuVLM (ours) & \textbf{48.9} {\tiny $\pm$ 1.53} & \textbf{73.4} {\tiny $\pm$ 1.66} & \textbf{48.4} {\tiny $\pm$ 1.31} & \textbf{25.1} {\tiny $\pm$ 1.17} & 9.45 {\tiny $\pm$ 0.16} & \textbf{39.1} {\tiny $\pm$ 0.66} & \textbf{73.5} {\tiny $\pm$ 21.1}   \\
  \bottomrule
  \end{tabular}
  }
  \end{center}
\end{small}
\end{table}

\subsection{Performance Comparison}
We present a comprehensive performance comparison between our proposed CurricuVLM and baseline methods across two evaluation phases, with results shown in Tables \ref{tab1}-\ref{tab2} and Figs.~\ref{exp-RL}-\ref{exp-CloseLoop}. Except for BC, all methods were trained for 1 million steps, with experiments repeated three times using different random seeds to ensure statistical reliability. Mean values and standard deviations are reported. Table \ref{tab1} presents the performance metrics under regular driving conditions where BVs follow their original trajectories, while Table \ref{tab2} shows the results in transformed safety-critical scenarios. Similarly, in Figs.~\ref{exp-RL}-\ref{exp-CloseLoop}, we illustrate the learning curves throughout the training process, where solid lines represent mean values and shaded regions indicate the standard deviations across three runs. For each figure, the top row demonstrates the performance under regular driving conditions, and the bottom row shows the corresponding results in safety-critical scenarios.

\subsubsection{Comparison with Traditional RL Methods}

\begin{figure}[!t]
  \centerline{\includegraphics[width=0.99\textwidth]{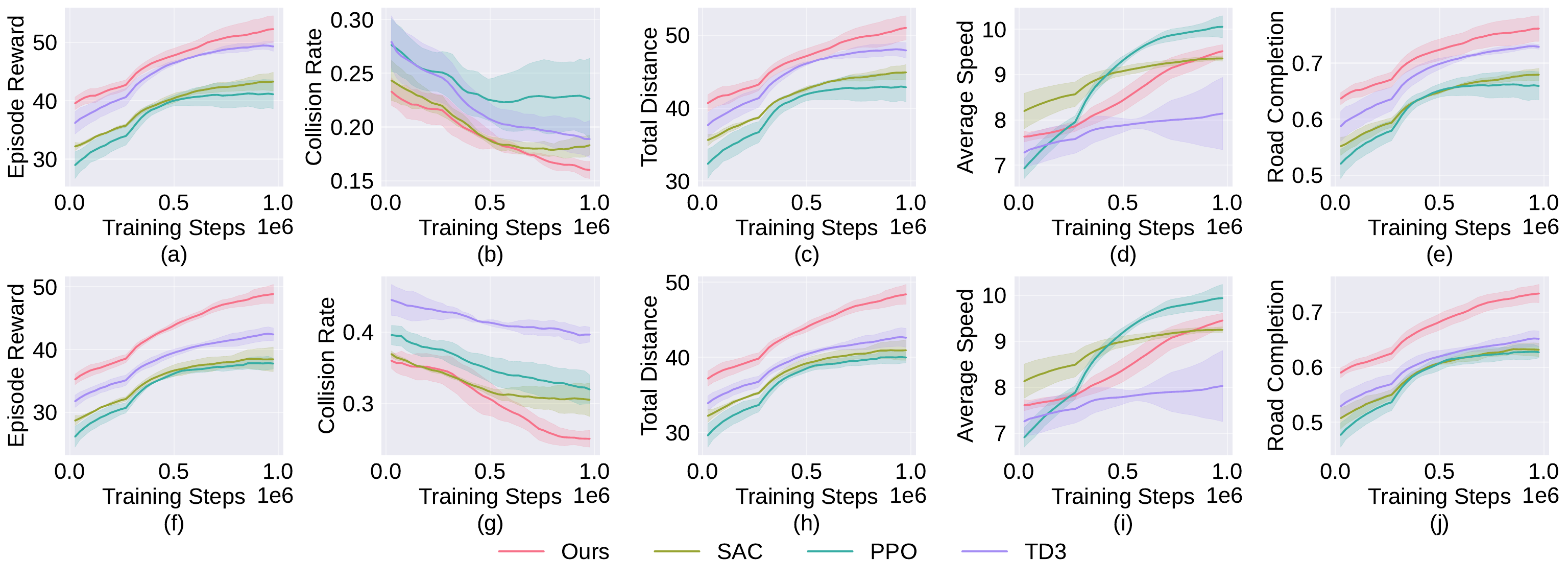}}
  \caption{Performance comparison with traditional RL baselines in both regular and safety-critical test scenarios.}
  \label{exp-RL}
\end{figure}

First, we analyze the performance comparison between CurricuVLM and traditional RL methods (SAC, PPO, and TD3) as shown in Fig.~\ref{exp-RL} and Tables~\ref{tab1}-\ref{tab2}. The top row in Fig.~\ref{exp-RL}(a)-(e) demonstrates that in regular driving scenarios, our method consistently outperforms all traditional RL baselines across key metrics. Specifically, while TD3 achieves the best performance among traditional RL methods with an episode reward of 49.3 and a road completion rate of 72.9\%, CurricuVLM further improves these metrics to 52.3 and 76.2\%, respectively. More importantly, in terms of safety metrics, our method reduces the crash rate to 16.0\% compared to SAC's 18.3\%, while maintaining a competitive average speed of 9.52 m/s.

The advantages of CurricuVLM become more pronounced in safety-critical scenarios, as shown in Fig.~\ref{exp-RL}(f)-(j). When traditional RL methods face challenging scenarios, their performance significantly degrades. As shown in Table~\ref{tab2}, the crash rates increase to 30.5\% for SAC, 32.0\% for PPO, and 39.7\% for TD3. In contrast, CurricuVLM maintains a substantially lower crash rate of 25.1\%. Additionally, our method achieves higher episode rewards of 48.9 compared to TD3's 42.4, along with superior road completion rates of 73.4\% compared to 65.2\% achieved by TD3. These results suggest that our approach effectively prepares the agent for handling complex driving situations without compromising basic driving capabilities.
Furthermore, CurricuVLM demonstrates superior learning effectiveness, with a failure-to-success rate of 39.1\% and a success-to-success rate of 73.5\%. These results significantly outperform TD3, which exhibits a failure-to-success rate of 28.6\% and a success-to-success rate of 64.3\%.
This indicates our method effectively learns from failed scenarios while maintaining stable performance across successful ones.

\begin{figure}[!t]
  \centerline{\includegraphics[width=0.99\textwidth]{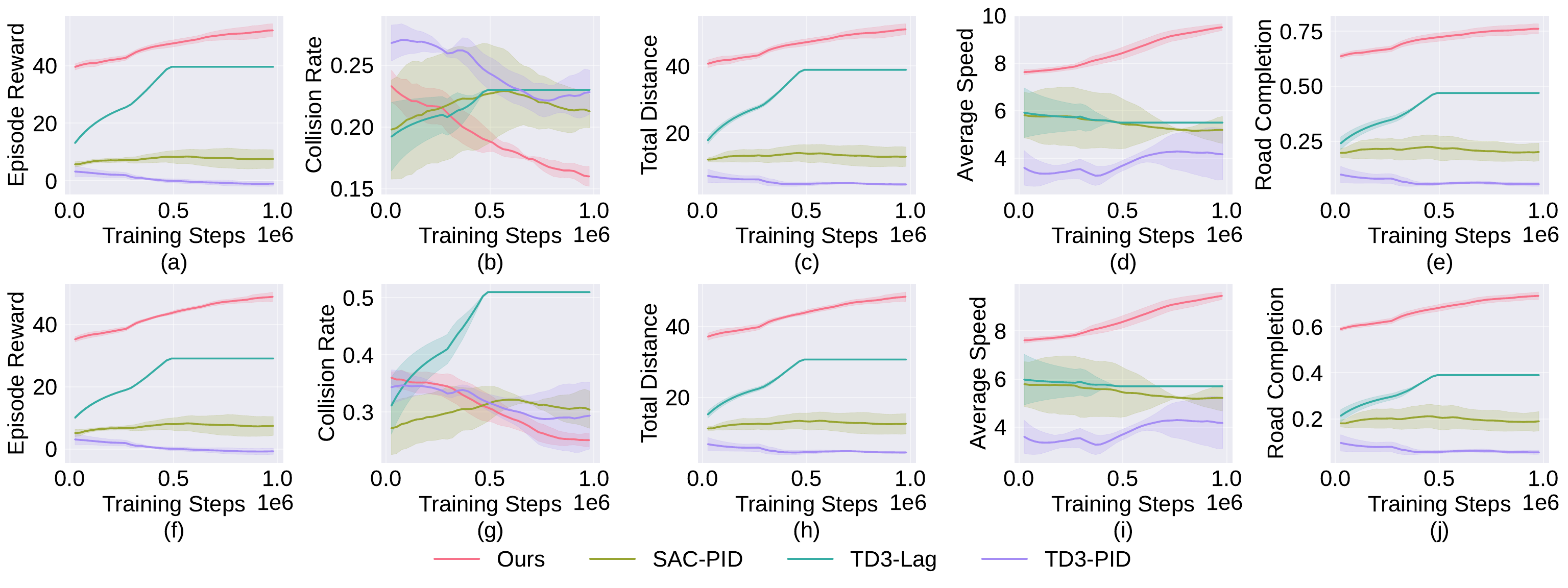}}
  \caption{Performance comparison with safe RL baselines in both regular and safety-critical test scenarios.}
  \label{exp-SafeRL}
\end{figure}

\subsubsection{Comparison with Safe RL Methods}

Next, we compare CurricuVLM with safe RL methods including SAC-PID, TD3-Lag, and TD3-PID. As shown in the top row of Fig.~\ref{exp-SafeRL}, we observe that safe RL methods exhibit overly conservative behavior in regular driving scenarios. According to Table~\ref{tab1}, these methods achieve limited episode rewards of 7.57, 39.6, and -1.0, respectively, along with low road completion rates ranging from 6.48\% to 46.9\%. Although they maintain relatively low crash rates comparable to CurricuVLM, this safety comes at a substantial cost to driving efficiency, as reflected in their low average speeds below 5.5 m/s and limited total distances traveled.

The limitations of safe RL methods become particularly evident when safety-critical scenarios are introduced. As shown in Table~\ref{tab2}, their episode rewards further drop below 30. More importantly, these methods fail to maintain safety despite their conservative nature, with crash rates increasing substantially to 30.4\% for SAC-PID and 51.0\% for TD3-Lag. This degradation can be attributed to their training process, which primarily focuses on maintaining safety margins in regular driving conditions without exposure to safety-critical scenarios. This is particularly evident in their consistently poor learning effectiveness, with all failure-to-success rates below 8.5\% in safety-critical scenarios, while CurricuVLM maintains a robust 39.1\%.
In contrast, CurricuVLM achieves superior performance through its personalized safety-critical curriculum learning approach. By systematically exposing the agent to increasingly challenging scenarios based on its current capabilities, our method develops robust driving policies that maintain both safety and efficiency. 
The smooth and consistent improvement shown in the learning curves demonstrates that our curriculum effectively bridges the gap between regular and safety-critical driving conditions, addressing a fundamental limitation of safe RL approaches that lack exposure to diverse safety-critical scenarios during training.

\begin{figure}[!t]
  \centerline{\includegraphics[width=0.99\textwidth]{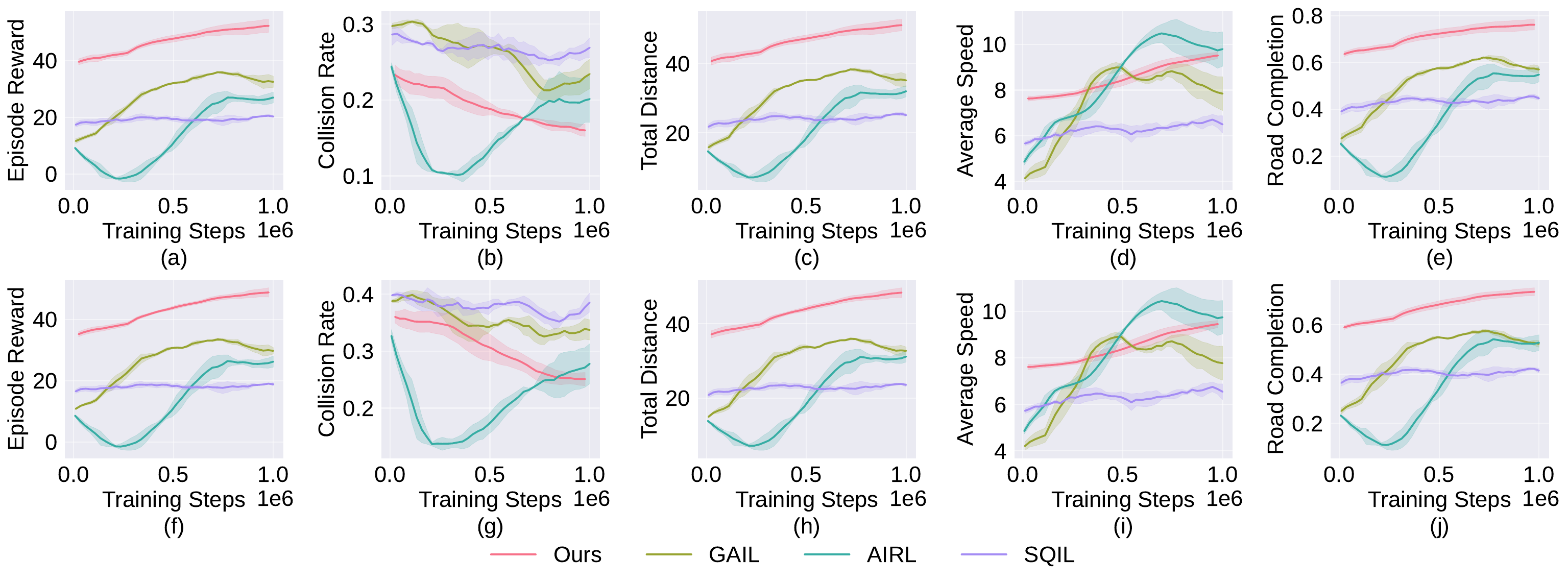}}
  \caption{Performance comparison with imitation learning baselines in both regular and safety-critical test scenarios.}
  \label{exp-Imitation}
\end{figure}

\subsubsection{Comparison with Imitation Learning Methods}

We further compare CurricuVLM with imitation learning methods including GAIL, AIRL, SQIL, and BC, as shown in Fig.~\ref{exp-Imitation} and Fig.~\ref{exp-BC}. For training these imitation learning algorithms, we collected 28K expert trajectories with 100\% success rate from the training scenarios. In Fig.~\ref{exp-Imitation}(a)-(e), we observe that imitation learning methods achieve moderate performance in regular driving scenarios. According to Table~\ref{tab1}, these methods obtain episode rewards ranging from 20.3 to 32.5 and road completion rates between 44.7\% and 58.6\%. 
The fundamental limitations of imitation learning approaches manifest more clearly in safety-critical scenarios. As demonstrated in Table~\ref{tab2}, their performance deteriorates significantly, with episode rewards dropping to between 18.9 and 29.8, and crash rates increasing to range from 27.8\% to 38.5\%. This degradation occurs because these methods struggle to generalize beyond the distribution of expert demonstrations, which only contain successful trajectories in regular driving conditions. 

\begin{figure}[!t]
  \centerline{\includegraphics[width=0.99\textwidth]{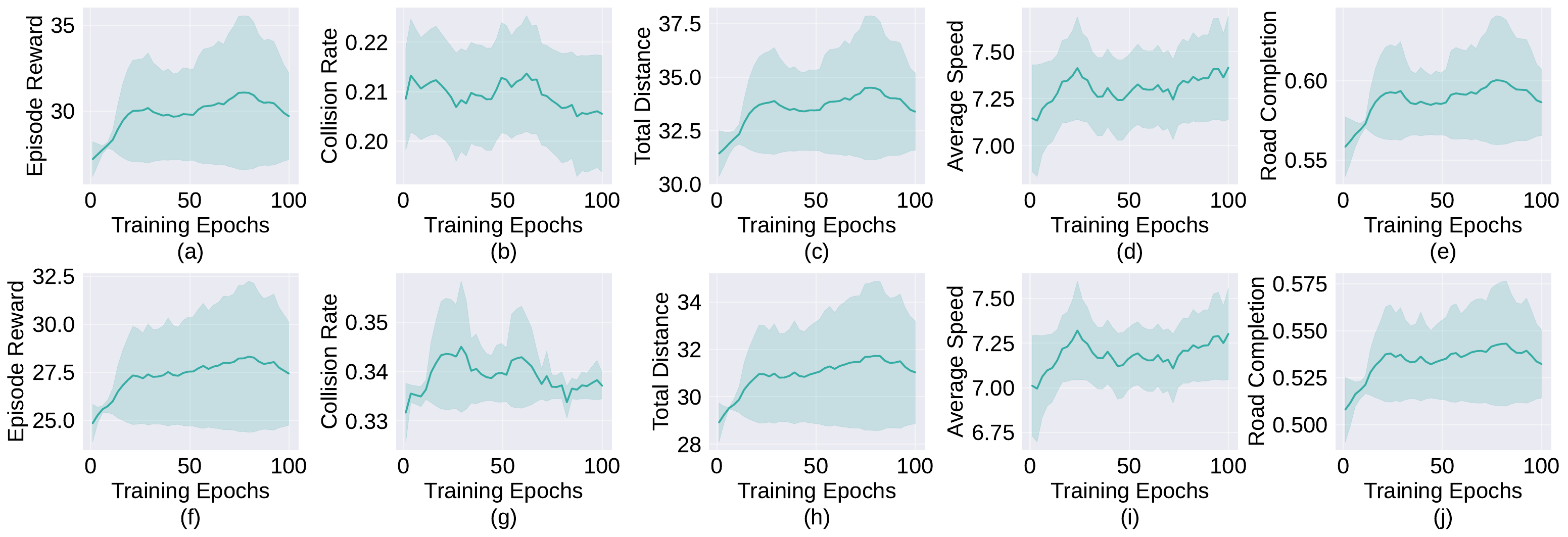}}
  \caption{Performance comparison with BC in both regular and safety-critical test scenarios.}
  \label{exp-BC}
\end{figure}

Due to the different training paradigm, BC is trained for 100 epochs and its results are shown separately in Fig.~\ref{exp-BC}. Despite having access to high-quality expert demonstrations, BC shows modest performance with an episode reward of 29.7 and road completion rate of 58.6\% in regular scenarios, which further decline to 27.4 and 53.2\% respectively in safety-critical scenarios. The learning curves in Fig.~\ref{exp-BC} reveal that BC quickly converges after the initial training phase, suggesting limited generalization capability beyond the demonstrated behaviors. In contrast, CurricuVLM demonstrates superior performance and adaptability through its curriculum learning approach. Rather than solely relying on expert demonstrations, our method actively explores and learns from diverse scenarios while maintaining safety awareness. This enables CurricuVLM to achieve substantially higher episode rewards and road completion rates while maintaining lower crash rates in both regular and safety-critical scenarios.

\begin{figure}
  \centerline{\includegraphics[width=0.99\textwidth]{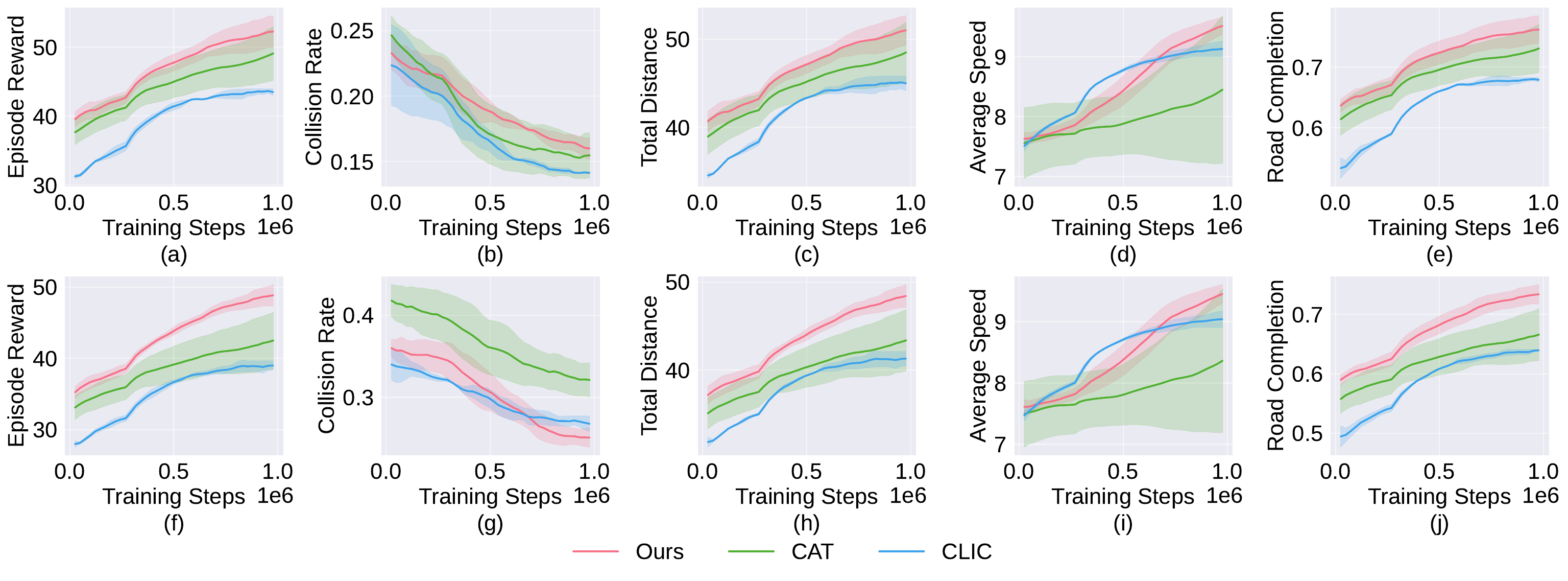}}
  \caption{Performance comparison with closed-loop curriculum learning methods in both regular and safety-critical test scenarios.}
  \label{exp-CloseLoop}
\end{figure}

\subsubsection{Comparison with Closed-loop Curriculum Learning Methods}

Finally, we compare CurricuVLM with state-of-the-art closed-loop curriculum learning methods, CAT and CLIC, as shown in Fig.~\ref{exp-CloseLoop}. In regular driving scenarios, both baseline methods demonstrate competitive performance. Particularly, CAT achieves strong navigation capabilities with high road completion rates, while CLIC maintains the lowest crash rates. However, CurricuVLM still outperforms these methods across key metrics including episode rewards and road completion rates while maintaining comparable safety levels.
When tested in safety-critical scenarios, our method's advantages in handling complex driving situations become evident. Both CAT and CLIC experience notable performance degradation, especially in terms of safety metrics and road completion rates. This suggests that their curriculum generation strategies, while effective for regular driving conditions, may not adequately prepare the agent for more challenging scenarios. In contrast, CurricuVLM maintains robust performance across all metrics, demonstrating its superior ability to handle safety-critical situations.
Moreover, CurricuVLM achieves higher failure-to-success and success-to-success rates in safety-critical scenarios compared to CAT and CLIC. These results validate the effectiveness of our VLM-based approach in identifying and addressing specific performance bottlenecks through personalized curriculum generation, leading to more robust and adaptable driving policies.

\begin{table}
\begin{small}
  \caption{Performance comparison with ablation variant models in the \textit{safety-critical test scenarios}. Mean and standard deviation over 3 seeds. The best results are marked in \textbf{bold}.}
  \label{tab3}
  \centering
  \begin{center}
  \renewcommand{\arraystretch}{1.5}
  \resizebox{\textwidth}{!}{
  \begin{tabular}{lccccccc}
  \toprule
  % \multirow{2}{*}{Category} & 
  \multirow{2}{*}{Experiment} &  \multirow{2}{*}{\shortstack{Episode\\Reward}} \multirow{2}{*}{$\uparrow$} & \multirow{2}{*}{\shortstack{Road\\Completion (\%)}} \multirow{2}{*}{$\uparrow$} & \multirow{2}{*}{\shortstack{Total\\Distance}} \multirow{2}{*}{$\uparrow$} & \multirow{2}{*}{\shortstack{Crash\\Rate (\%)}} \multirow{2}{*}{$\downarrow$} & \multirow{2}{*}{\shortstack{Average\\Speed}} \multirow{2}{*}{$\uparrow$} & \multirow{2}{*}{\shortstack{Failure-to-\\Success Rate (\%)}} \multirow{2}{*}{$\uparrow$} & \multirow{2}{*}{\shortstack{Success-to-\\Success Rate (\%)}} \multirow{2}{*}{$\uparrow$} \\
  & & & & & & & \\
  \midrule
  
   W/o Curriculum & 42.4 {\tiny $\pm$ 1.01} & 65.2 {\tiny $\pm$ 1.40} & 42.6 {\tiny $\pm$ 1.26} & 39.7 {\tiny $\pm$ 1.04} & 8.02 {\tiny $\pm$ 0.77}   & 28.6 {\tiny $\pm$ 2.79} & 64.3 {\tiny $\pm$ 21.4} \\
   W/o Online Generation  & 42.4 {\tiny $\pm$ 2.37} & 66.8 {\tiny $\pm$ 4.20} & 43.4 {\tiny $\pm$ 2.00} & 31.5 {\tiny $\pm$ 3.45} & 8.70 {\tiny $\pm$ 0.80} & 33.5 {\tiny $\pm$ 2.38} & 47.5 {\tiny $\pm$ 22.5}  \\
   W/o Learned Prior  & 38.3 {\tiny $\pm$ 4.18} & 60.8 {\tiny $\pm$ 5.04} & 40.2 {\tiny $\pm$ 3.81} & 29.4 {\tiny $\pm$ 3.04} & 8.00 {\tiny $\pm$ 1.37} & 25.4 {\tiny $\pm$ 7.50} & 41.7 {\tiny $\pm$ 27.9}  \\
   W/o VLM Analysis &  43.2 {\tiny $\pm$ 0.32} & 67.6 {\tiny $\pm$ 0.67} & 44.1 {\tiny $\pm$ 0.32} & 32.6 {\tiny $\pm$ 0.50} & 8.79 {\tiny $\pm$ 0.71}  & 32.2 {\tiny $\pm$ 0.42} & 61.3 {\tiny $\pm$ 1.25}  \\
  \rowcolor{gray!15}
  CurricuVLM (full) & \textbf{48.9} {\tiny $\pm$ 1.53} & \textbf{73.4} {\tiny $\pm$ 1.66} & \textbf{48.4} {\tiny $\pm$ 1.31} & \textbf{25.1} {\tiny $\pm$ 1.17} & \textbf{9.45} {\tiny $\pm$ 0.16} & \textbf{39.1} {\tiny $\pm$ 0.66} & \textbf{73.5} {\tiny $\pm$ 21.1}   \\
  \bottomrule
  \end{tabular}
  }
  \end{center}
\end{small}
\end{table}

\subsection{Ablation Studies}

To systematically evaluate the contribution of each component in CurricuVLM, we conduct comprehensive ablation studies through four variant models and assess their performance in safety-critical testing scenarios. The variants are designed to examine the impact of our key technical components:
\begin{itemize}
 \item \textbf{W/o VLM Analysis}: Instead of using VLM-based behavioral analysis, this variant directly utilizes the object trajectories from the Waymo Open Motion Dataset to generate scenarios targeting high-probability collision points with the AV, removing the personalized aspect of scenario generation.

 \item \textbf{W/o Learned Prior}: This variant replaces our DenseTNT-based trajectory generation approach with a rule-based approach that creates simplified adversarial trajectories through waypoint selection and Bezier curve fitting, testing the importance of realistic trajectory modeling.

 \item \textbf{W/o Online Generation}: Rather than dynamically generating personalized scenarios based on the agent's evolving behavior patterns, this variant pre-generates a fixed set of safety-critical scenarios and injects them into training, eliminating the adaptive nature of our curriculum.

 \item \textbf{W/o Curriculum}: This variant removes the dynamic curriculum scheduling mechanism, training the agent without the structured exposure to safety-critical scenarios, thus testing the importance of our probabilistic scheduling approach.
\end{itemize}

The performance comparison of these variants against our full model is presented in Table \ref{tab3}. The ablation results demonstrate the significant contribution of each component in CurricuVLM. The W/o VLM Analysis variant shows substantial performance degradation, with the episode reward decreasing from 48.9 to 43.2 and the road completion rate falling from 73.4\% to 67.6\%. This degradation is particularly evident in the failure-to-success rate, which decreases by 6.9\%, highlighting VLM's crucial role in understanding and addressing agent failures. The W/o Learned Prior variant exhibits the poorest performance among all variants, with episode rewards dropping to 38.3 and crash rates increasing significantly. This result indicates the importance of learned trajectory generation in creating realistic and challenging scenarios. The W/o Online Generation variant shows limited adaptability with lower success-to-success rates of 47.5\%, indicating that pre-generated scenarios without dynamic adaptation fail to effectively address the agent's evolving behavioral patterns. This validates our design choice of maintaining a rolling buffer of agent behaviors. The W/o Curriculum variant demonstrates inferior performance across all metrics, particularly in safety-critical scenarios where the crash rate increases by 14.6\%. This substantial degradation validates our probabilistic curriculum scheduling mechanism in balancing regular training and safety-critical scenario exposure.

\subsection{Sensitivity Analysis}

\subsubsection{Impact of Vision-Language Model Selection}

\begin{figure}
  \centerline{\includegraphics[width=0.99\textwidth]{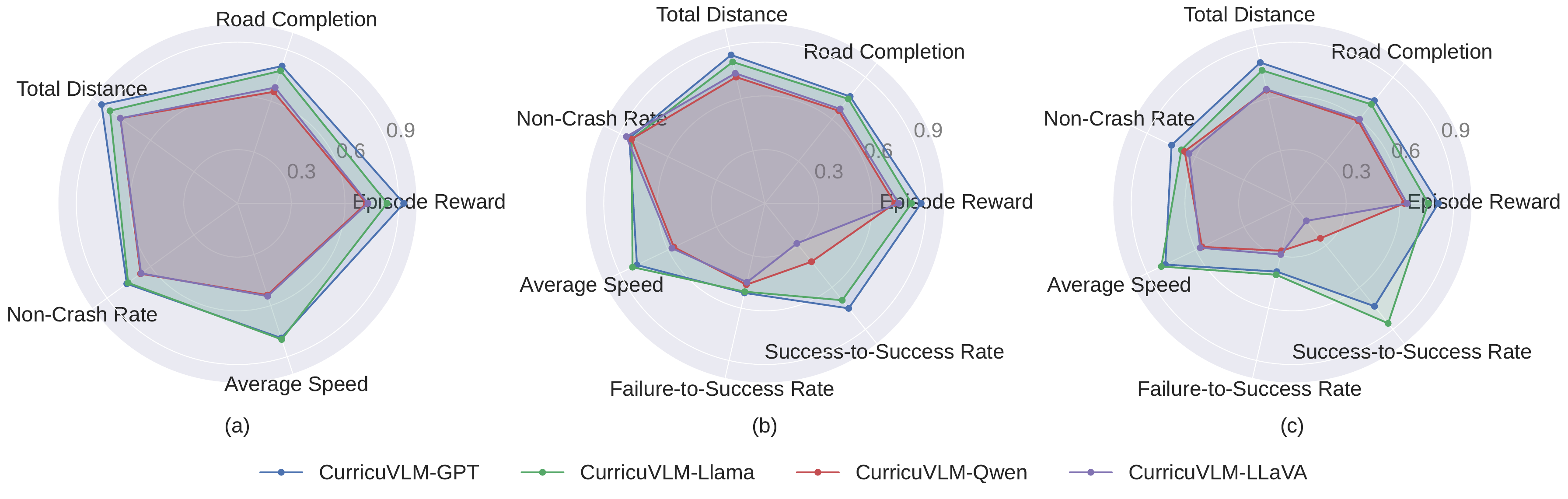}}
  \caption{Performance comparison with our CurricuVLM using different VLMs in: (a) training scenarios, (b) regular test scenarios, and (c) safety-critical test scenarios.}
  \label{exp-vlm-radarchart}
\end{figure}

The performance of CurricuVLM may be sensitive to the choice of the underlying Vision-Language Model. To quantify this sensitivity, we conducted experiments with four state-of-the-art VLMs: GPT-4o \citep{achiam2023gpt}, Llama-3.2-11B-Vision \citep{dubey2024llama}, Qwen2-VL-7B \citep{wang2024qwen2}, and LLaVA-1.5-13B \citep{liu2024visual}. Fig.~\ref{exp-vlm-radarchart} presents radar charts comparing their performance across multiple metrics, where all indicators are normalized to [0,1] with higher values indicating better performance. The results show that CurricuVLM-GPT and CurricuVLM-Llama exhibit comparable superior performance across all evaluation metrics, with their performance profiles closely matching each other. In contrast, CurricuVLM-Qwen and CurricuVLM-LLaVA show relatively lower performance across most metrics, particularly in episode rewards and road completion rates.

The varying performance across different VLMs can be attributed to several factors. CurricuVLM-GPT and CurricuVLM-Llama benefit from more sophisticated VLM architectures and extensive pre-training on diverse visual data, enabling a better understanding of complex driving scenarios. Their superior reasoning capabilities allow for more nuanced analysis of agent behavior and more effective curriculum generation. Interestingly, despite LLaVA-1.5-13B having a larger parameter count than Llama-3.2-11B, it shows comparatively lower performance. This suggests that model size alone may not be the primary factor in determining effectiveness, as the quality of pre-training data and the model's architectural design for processing driving-specific scenarios likely play significant roles. Another noteworthy observation is the consistency in performance patterns across different evaluation settings. While CurricuVLM-Qwen and CurricuVLM-LLaVA achieve reasonable non-crash rates in training scenarios, they show limitations in generalizing to new scenarios, as reflected in their lower success-to-success and failure-to-success rates. This indicates that these models may struggle with the complex task of analyzing driving behavior and generating appropriate curriculum adjustments, despite their fundamental ability to process visual information.

\subsubsection{Analysis of Batch Size for Safety-Critical Event Analysis}

\begin{figure}
  \centerline{\includegraphics[width=0.99\textwidth]{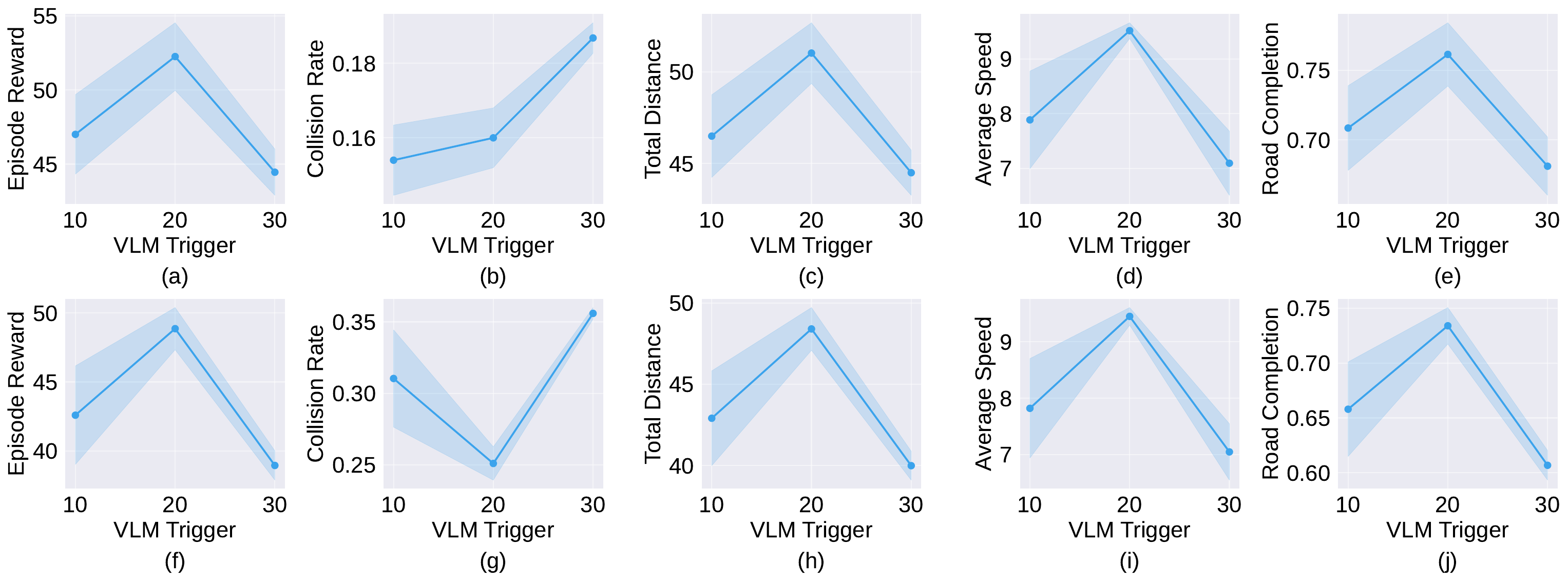}}
    \caption{Performance comparison under different batch sizes $N$ for safety-critical event analysis in both regular and safety-critical test scenarios.}
  \label{exp-trigger}
\end{figure}

We further investigate the impact of different batch sizes $N$ in our safety-critical event analysis process on the performance of CurricuVLM. As described in Section \ref{sec4.2}, the framework adopts a batch-wise analysis strategy by accumulating $N$ safety-critical events before initiating the GPT-4o analysis phase. Our experiments examine three different batch sizes ($N=10$, $20$, and $30$) to understand their effects on both regular and safety-critical scenario performance.

Fig.~\ref{exp-trigger} illustrates how different batch sizes affect various performance metrics. In regular driving scenarios, as shown in Fig.~\ref{exp-trigger}(a)-(e), we observe that a batch size of $N=20$ yields the best overall performance. With this setting, the model achieves the highest episode reward, optimal total distance traveled, and best road completion rate. While this setting shows a slight increase in collision rate compared to $N=10$, this marginal safety trade-off is compensated by significantly improved driving efficiency. This observation aligns with our framework's design principle of balancing safety and operational efficiency. The experimental results in safety-critical scenarios, as shown in Fig.~\ref{exp-trigger}(f)-(j), further validate the advantages of $N=20$. A smaller batch size of $N=10$ leads to overly frequent analyses and more conservative behavior, resulting in suboptimal driving efficiency. Conversely, a larger batch size of $N=30$ delays the behavioral pattern analysis, leading to compromised safety performance with elevated collision rates of 0.35.

These findings demonstrate that the batch size in safety-critical event analysis significantly influences the framework's performance. An insufficient batch size may lead to premature pattern recognition and cautious behavior, while an excessive batch size might miss crucial behavioral patterns and delay necessary curriculum adaptations. Our experiments indicate that $N=20$ achieves an optimal balance, allowing sufficient event accumulation for meaningful pattern recognition while ensuring timely curriculum updates.

\begin{figure}
  \centerline{\includegraphics[width=0.99\textwidth]{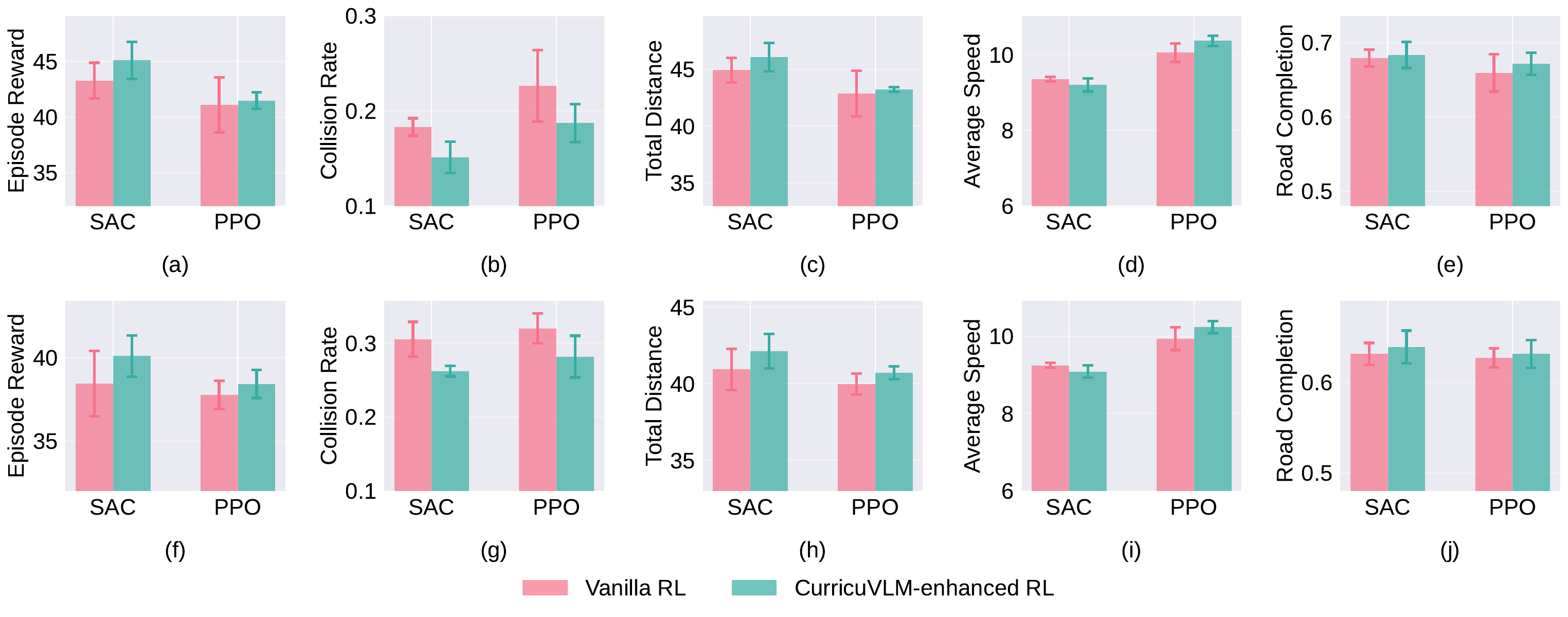}}
  \caption{Performance comparison with CurricuVLM using different RL algorithms in both regular and safety-critical test scenarios.}
  \label{exp-ppo-sac-bar}
\end{figure}

\subsection{Compatibility with Different RL Algorithms}\label{sec5.6}

To demonstrate the versatility of our proposed CurricuVLM framework, we evaluate its performance when integrated with different RL algorithms, specifically SAC and PPO, in addition to the default TD3 implementation. Fig.~\ref{exp-ppo-sac-bar} presents a comprehensive comparison between the vanilla versions of these algorithms and their CurricuVLM-enhanced counterparts across multiple performance metrics.

In regular driving scenarios, as shown in Fig.~\ref{exp-ppo-sac-bar}(a)-(e), CurricuVLM consistently improves the performance of both base algorithms. For SAC, our framework substantially improves the episode reward and total distance traveled while reducing the collision rate from 0.18 to 0.15. The enhanced SAC maintains a similar road completion rate, suggesting that the safety improvements do not come at the cost of task completion. When applied to PPO, CurricuVLM also achieves notable improvements in both episode reward and safety metrics, while the increased total distance traveled demonstrates enhanced driving efficiency.
In safety-critical scenarios, shown in Fig.~\ref{exp-ppo-sac-bar}(f)-(j), both enhanced algorithms maintain their performance advantages. The CurricuVLM-enhanced SAC demonstrates improved episode rewards and reduced collision rates while achieving longer travel distances. Similarly, the enhanced PPO exhibits better performance across all metrics, particularly in safety-related measurements. These results indicate that our framework effectively enhances the robustness of both algorithms in challenging driving conditions.

The consistent performance improvements across different RL algorithms suggest that our VLM-based curriculum learning framework provides a generalizable approach to enhancing autonomous driving systems. Rather than being algorithm-specific, the benefits of personalized scenario analysis and adaptive curriculum generation appear to be fundamental improvements that can complement various underlying RL methodologies. This compatibility highlights the potential of our approach as a general framework for improving the safety and efficiency of autonomous driving systems.

\begin{figure}
  \centerline{\includegraphics[width=0.99\textwidth]{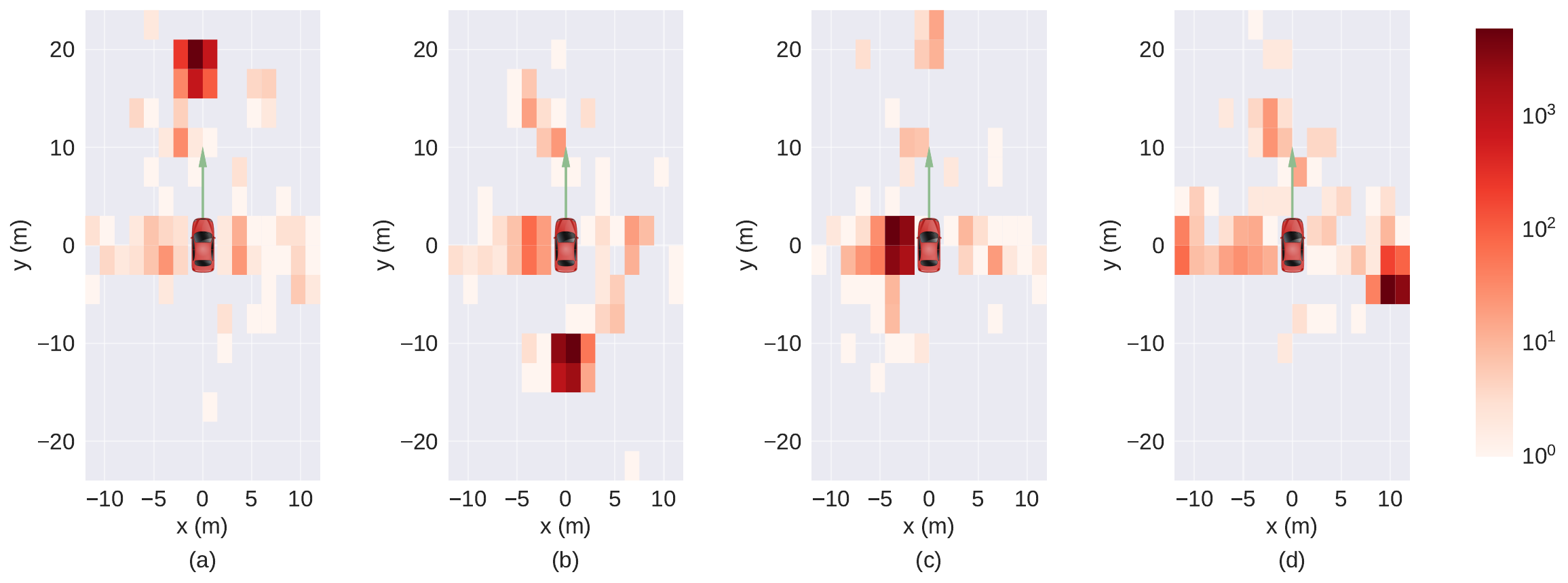}}
\caption{Spatial distribution of selected BV positions for the AV agent with different disabled perception regions: (a) forward, (b) backward, (c) left, and (d) right blind spots. The red car represents the ego vehicle with the green arrow indicating its heading direction.}
  \label{fig_heatmap}
\end{figure}

\subsection{Analysis of Personalized Curriculum Generation}

To validate the effectiveness of CurricuVLM in generating personalized safety-critical scenarios that align with specific agent weaknesses, we conduct controlled experiments by introducing various perceptual limitations to the AV agent. By deliberately modifying the agent's LiDAR observations in different directions (forward, backward, left, and right), we create agents with known behavioral vulnerabilities that should be detected by our VLM-based analysis framework and addressed through curriculum generation. Fig.~\ref{fig_heatmap} visualizes the spatial distribution of the BV selected through Eq.~\eqref{eq8}. The ego vehicle is positioned at the origin (0, 0) and oriented upward, as indicated by the green arrow. The intensity of red coloring represents the logarithmic frequency of selected vehicle positions relative to the ego vehicle, with darker shades indicating higher frequencies.

The results reveal that our framework successfully aligns generated scenarios with agent-specific weaknesses. When analyzing agents with disabled forward perception, our CurricuVLM predominantly generates insights focused on forward regions. As shown in Fig.~\ref{fig_heatmap}(a), this translates into a high frequency of selected BV positions in front of the ego vehicle, indicated by the dark red regions at $y > 0$. The framework demonstrates similar alignment patterns for agents with other disabled sensing regions. Fig.~\ref{fig_heatmap}(b) shows a concentration of BV selections behind the ego vehicle when rear perception is disabled. For agents with disabled left and right perception, Fig.~\ref{fig_heatmap}(c) and (d) exhibit a clear pattern of high-frequency selections in the corresponding regions.

These experimental results validate CurricuVLM's ability to adapt to agent-specific limitations through personalized curriculum generation. The clear spatial correlation between disabled perceptual regions and BV position selections demonstrates that our framework successfully translates VLM-generated insights into meaningful training scenarios. This personalized curriculum generation capability marks a significant advancement over conventional approaches that rely on fixed scenario distributions or random sampling, enabling more efficient and focused improvement of autonomous driving capabilities.

\begin{figure}
  \centerline{\includegraphics[width=0.99\textwidth]{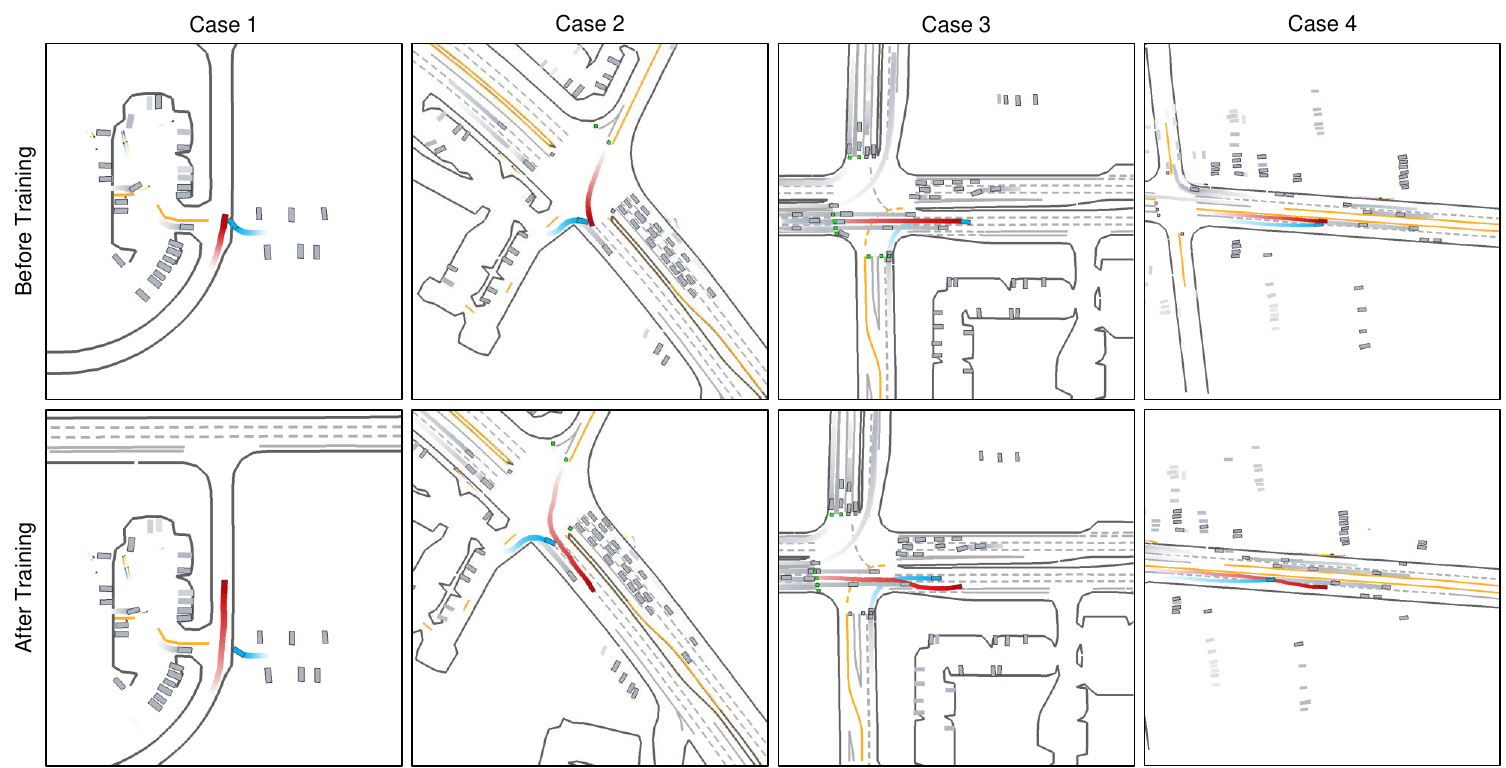}}
  \caption{Visualization of RL agent behavior before and after training with CurricuVLM in representative safety-critical scenarios. Red trajectories represent our AV agent, and blue trajectories indicate vehicles involved in potential collisions.}
  \label{cases}
\end{figure}

\subsection{Qualitative Analysis of Performance Improvement}

To provide deeper insights into how CurricuVLM improves driving behavior, we analyze four representative scenarios where the agent initially failed but successfully learned safe navigation strategies after training, as shown in Fig.~\ref{cases}. These cases demonstrate the agent's acquisition of sophisticated driving skills across diverse traffic situations.

The first two cases present the agent's improved handling of intersection-related challenges. 
In Case 1, when encountering an unprotected right-turning vehicle, the untrained agent fails to maintain a safe distance and collides with the turning vehicle. After training with CurricuVLM, the agent learns to proactively adjust its speed and lateral position, creating a larger safety margin while maintaining efficient progress. Similarly, in Case 2, the agent develops more sophisticated merging behavior, learning to select safer lane positions and adjust its speed to avoid potential conflicts with surrounding vehicles.

The method also demonstrates enhanced capabilities in handling dynamic traffic situations. Case 3 showcases the agent's improved response to sudden speed changes from leading vehicles. While the untrained agent fails to react appropriately to a decelerating vehicle, the trained agent successfully executes a safe lane change maneuver to maintain smooth traffic flow. This behavioral improvement extends to Case 4, where the agent learns to anticipate and respond to potential lane change conflicts. When confronted with an adjacent vehicle's lane change, the trained agent demonstrates sophisticated decision-making by accelerating and shifting to a safer lane while maintaining appropriate following distances.
These qualitative results indicate that CurricuVLM enables the agent to develop sophisticated driving behaviors beyond simple collision avoidance. 

\section{Conclusions}\label{sec6}

In this paper, we presented CurricuVLM, a novel framework that leverages Vision-Language Models to enable personalized curriculum learning for autonomous driving agents. Our approach addresses a critical challenge in autonomous driving: effectively incorporating safety-critical scenarios into policy learning while adapting to an agent's evolving capabilities. Through the innovative combination of VLMs' visual understanding capabilities and GPT-4o's systematic behavioral analysis, our framework can comprehensively analyze unsafe driving situations, identify critical behavioral patterns, and generate tailored training scenarios. The dynamic curriculum scheduling mechanism ensures these scenarios are integrated effectively into the training process, creating a closed-loop learning system that continuously adapts to the agent's evolving capabilities. Extensive experiments on the Waymo Open Motion Dataset demonstrate that CurricuVLM achieves superior performance across both regular and safety-critical scenarios compared to various state-of-the-art baselines. Furthermore, our experiments with different RL algorithms show that CurricuVLM serves as a general approach for enhancing autonomous driving systems across different learning paradigms.

While our current results are promising, several exciting directions remain for future research. First, our current implementation utilizes general-purpose VLMs for scene understanding and behavioral analysis. Fine-tuning these models on autonomous driving-specific data could potentially improve their understanding of traffic patterns, vehicle interactions, and safety-critical situations, leading to more precise and relevant behavioral insights. Second, incorporating Retrieval-Augmented Generation (RAG) techniques could enable our framework to reference traffic rules, driving regulations, and expert knowledge during behavioral analysis. This would allow the system to provide recommendations that are not only based on observed patterns but also grounded in established driving principles and safety guidelines. Third, developing methods that can directly generate realistic and controllable scenarios from VLM's natural language recommendations could streamline the curriculum adaptation process and potentially capture more nuanced aspects of the behavioral insights. This could involve exploring generative models that can bridge the gap between natural language descriptions and executable driving scenarios.

\section*{Acknowledgment}
This work was supported by the University of Wisconsin-Madison's Center for Connected and Automated Transportation (CCAT), a part of the larger CCAT consortium, a USDOT Region 5 University Transportation Center funded by the U.S. Department of Transportation, Award \#69A3552348305. The contents of this paper reflect the views of the authors, who are responsible for the facts and the accuracy of the data presented herein, and do not necessarily reflect the official views or policies of the sponsoring organization.

\appendix

\section{Detailed Algorithmic Procedures of CurricuVLM}\label{appendix:algorithms}

\begin{algorithm}[H]
\caption{Generate Personalized Safety-Critical Scenario}
\label{alg:generate}
\begin{algorithmic}[1]
\Require Current insights $\mathcal{I}$, regular scenario $\xi_{\text{reg}}$, behavior buffer $\mathcal{B}$, trajectory prior model $f_{\text{prior}}$
\State Initialize best score $s^* \leftarrow -\infty$ 
\State Initialize best trajectory $Y^{\text{BV}*} \leftarrow \texttt{None}$
\For{each background vehicle in $\xi_{\text{reg}}$} \Comment{Iterate through all BVs}
    \State Extract historical information $X \leftarrow (M, S^{\text{AV}}_{1:t}, S^{\text{BV}}_{1:t})$
    \State Generate trajectory predictions and probabilities $\{(Y^{\text{BV}}_i, p^{\text{BV}}_i)\}_{i=1}^K \leftarrow f_{\text{prior}}(X)$
    \For{$i = 1$ to $K$}
        \State Initialize scenario score $s_i \leftarrow 0$
        \For{$(s_1^j, a_1^j, p_1^j, ..., s_T^j, a_T^j, p_T^j)$ in $\mathcal{B}$}
            \For{$t = 1$ to $T$}
                \State Get BV position $pos_t^{\text{BV}}$ and bounding box $bbox_t^{\text{BV}}$ from $Y^{\text{BV}}_i$
                \State Get AV position $pos_t^{\text{AV}}$ and bounding box $bbox_t^{\text{AV}}$ from $(s_t^j, a_t^j)$
                \State Compute relative position $pos_t^{\text{rel}} \leftarrow \texttt{GetRelativePosition}(pos_t^{\text{BV}}, pos_t^{\text{AV}})$
                \If{$bbox_t^{\text{BV}} \cap bbox_t^{\text{AV}} \neq \emptyset$ \textbf{and} $pos_t^{\text{rel}} \in \mathcal{I}.\texttt{CriticalArea}$}
                \Comment{Eq.~(\ref{eq9})}
                    \State $s_i \leftarrow s_i + p^{\text{BV}}_i \cdot p_t^j \cdot \lambda^t$
                    \Comment{Eq.~(\ref{eq10})}
                \EndIf
            \EndFor
        \EndFor
        \If{$s_i > s^*$}
            \State $s^* \leftarrow s_i$
            \State $Y^{\text{BV}*} \leftarrow Y^{\text{BV}}_i$ \Comment{Eq.~(\ref{eq11})}
        \EndIf
    \EndFor
\EndFor
\State Construct safety-critical scenario $\xi_{\text{safe}}$ using $Y^{\text{BV}*}$ \Comment{Generate final scenario}
\State \Return $\xi_{\text{safe}}$
\end{algorithmic}
\end{algorithm}

\begin{algorithm}[H]
\caption{CurricuVLM Training Process}
\label{alg:curricuVLM}
\begin{algorithmic}[1]
\Require Initial policy $\pi_0$, VLM model $f_{\text{VLM}}$, GPT-4o model $g_{\text{GPT-4o}}$, prompt template $P(\cdot)$, task description $l$, batch size $N$, buffer size $L$, total steps $T$, regular scenarios pool $\mathcal{S} = \{\xi_1, \xi_2, ..., \xi_M\}$
\State Initialize empty event description buffer $\mathcal{D} \leftarrow \{\}$
\State Initialize behavior buffers $\{\mathcal{B}_m\}_{m=1}^M \leftarrow \{\emptyset\}^M$
\Comment{Store recent agent behaviors per scenario}
\State Initialize VLM insights $\mathcal{I} \leftarrow \texttt{None}$
\State Initialize curriculum $\mathcal{C} \leftarrow \{\}$
\State Initialize training step $t \leftarrow 1$
\While{$t \leq T$}
    \State Sample regular scenario $\xi_{\text{reg}}$ from $\mathcal{S}$ with index $m$
    \State Sample $p \sim \mathcal{U}(0,1)$
    \If{$p < \min(\kappa \cdot \frac{t}{T} \cdot p_{\text{max}}, p_{\text{max}})$ and $\mathcal{I} \neq \texttt{None}$}
        \State Generate personalized safety-critical scenario $\xi_{\text{safe}}$ based on Alg.~\ref{alg:generate} \Comment{Sec.~\ref{sec4.3}}
        \State Update curriculum $\mathcal{C} \leftarrow \mathcal{C} \cup \{\xi_{\text{safe}}\}$
        \State Execute $\xi_{\text{safe}}$ and collect trajectory $\tau$
    \Else
        \State Update curriculum $\mathcal{C} \leftarrow \mathcal{C} \cup \{\xi_{\text{reg}}\}$
        \State Execute $\xi_{\text{reg}}$ and collect trajectory $\tau$
    \EndIf
    \State Update behavior buffer $\mathcal{B}_m$ from $\tau$
    \If{$|\mathcal{B}_m| > L$}
        \State Remove oldest trajectory from $\mathcal{B}_m$
    \EndIf
    \If{safety-critical event occurs in $\tau$} \Comment{Sec.~\ref{sec4.2}}
        \State Generate description $d \leftarrow f_{\text{VLM}}(l, \mathcal{O}^{\leq k})$ \Comment{Eq.~(\ref{eq3})}
        \State Add to buffer $\mathcal{D} \leftarrow \mathcal{D} \cup \{d\}$
        \If{$|\mathcal{D}| = N$}
            \State Update insights $\mathcal{I} \leftarrow g_{\text{GPT-4o}}(P(\mathcal{D}))$ \Comment{Eq.~(\ref{eq4})}
            \State Clear buffer $\mathcal{D} \leftarrow \{\}$
        \EndIf
    \EndIf
    \State Update policy $\pi_t$ using standard RL update
    \State $t \leftarrow t + |\tau|$ \Comment{Add length of collected trajectory}
\EndWhile
\State \Return Final policy $\pi_T$
\end{algorithmic}
\end{algorithm}

\section{Theoretical Analysis}\label{theory}

In this appendix, we provide theoretical analysis for our CurricuVLM framework. Specifically, we analyze the convergence properties of our approach and show that, with sufficient training, the learned policy can achieve near-optimal performance across both regular and safety-critical scenarios. The key insight is that our VLM-based behavioral analysis combined with dynamic curriculum scheduling ensures comprehensive coverage of critical learning scenarios while maintaining basic driving competencies.

\begin{theorem}\label{thm:curric_theorem}
Suppose the VLM-based analysis accurately provides a consistent and unbiased estimate of the agent’s performance deficiencies across both regular and safety-critical scenarios. Let $\pi^*$ be the optimal policy that solves all tasks in the curriculum $\mathcal{C}$, and let $\hat{\pi}_t$ be the policy learned at iteration $t$ under our CurricuVLM framework. Then, as $t \to \infty$, $\hat{\pi}_t$ converges to $\pi^*$ with high probability. In particular, the dynamic scheduling mechanism ensures that the agent progressively improves performance on both regular and safety-critical scenarios, ultimately converging to a policy that minimizes safety violations while maintaining driving efficiency.
\end{theorem}

\begin{proof}
Let $p_r$ and $p_s$ denote the probabilities of sampling regular and safety-critical scenarios respectively under dynamic scheduling, where $p_r + p_s = 1$. The scheduling mechanism increases $p_s$ over training steps while maintaining a positive upper bound $p_{\text{max}}<1$, ensuring infinite sampling of both scenario types in the limit.
The VLM-based analysis provides an unbiased measure of the agent's behavioral deficiencies. 
Therefore, if $E_t$ denotes the set of safety-critical events up to iteration $t$, the VLM analysis identifies with probability at least $\beta>0$ any recurring deficiency in $E_t$ that degrades performance.
According to standard results in RL \citep{sutton2018reinforcement}, the iterative policy improvement process converges when sufficient coverage over all relevant states and tasks is guaranteed. Our framework ensures such coverage through the gradually adjusted $p_s$ of sampling safety-critical scenarios.

In each training iteration, the policy updates rely on experiences collected from both regular and safety-critical tasks. Let $\mathcal{M}(\hat{\pi}_t,\mathcal{I})$ represent the non-negative expected deficiency measure of policy $\hat{\pi}_t$ given VLM-generated insights $\mathcal{I}$. 
Because the scheduling mechanism guarantees repeated exposure to scenarios provoking the deficiency while also preserving exposure to regular tasks, $\mathcal{M}(\hat{\pi}_t,\mathcal{I})$ must decrease over time.  This improvement process can be expressed as follows:
\begin{equation}
\mathcal{M}(\hat{\pi}_{t+1},\mathcal{I}) \;\le\; \mathcal{M}(\hat{\pi}_t,\mathcal{I}) \;-\; \alpha \, \bigl( \mathcal{M}(\hat{\pi}_t,\mathcal{I}) - \mathcal{M}^* \bigr),
\end{equation}
where $\alpha>0$ is a learning rate, and $\mathcal{M}^*$ is the minimal deficiency level under $\pi^*$. Iterating this inequality shows that $\hat{\pi}_t$ converges to a policy whose expected deficiency is arbitrarily close to $\mathcal{M}^*$, reflecting convergence to $\pi^*$. 
Hence, given sufficient training steps and a guaranteed non-vanishing $p_s$, policy $\hat{\pi}_t$ continuously improves in safety-critical performance while retaining competency in regular scenarios. This iterative process converges to a policy that achieves high safety and efficiency across all scenarios in $\mathcal{C}$.
\end{proof}

\bibliographystyle{elsarticle-harv} 
\biboptions{authoryear}
\bibliography{reference}

\end{document}